%% file: main.tex
\documentclass{article}

\usepackage{graphicx}
\usepackage{latexsym,amsmath,amssymb,amsfonts,amstext,amsthm,mathrsfs,bm}
\usepackage[utf8]{inputenc} 
\usepackage[T1]{fontenc}    

\usepackage[colorlinks = true]{hyperref}        
\hypersetup{linkcolor = red, citecolor = blue}

\usepackage{cleveref}

\usepackage{url}            
\usepackage{booktabs}       
\usepackage{amsfonts}       
\usepackage{nicefrac}       
\usepackage{microtype}      
\usepackage{breakcites}

\usepackage{epsfig}
\usepackage{caption}
\usepackage{dsfont}
\usepackage{enumitem}
\usepackage{thm-restate}
\usepackage{color}
\usepackage{mathtools}
\usepackage{bbm}  
\usepackage{cite}  

\usepackage{mathtools}

\usepackage[round]{natbib}

\usepackage{apalike}

\usepackage{geometry}
\hypersetup{
	colorlinks,linkcolor=red,anchorcolor=blue,citecolor=blue}
\usepackage{multirow}

\usepackage{authblk}

\makeatletter
\newcommand{\printfnsymbol}[1]{%
	\textsuperscript{\@fnsymbol{#1}}%
}
\makeatother
 
\usepackage{algorithm,algcompatible,amsmath}
\algnewcommand\INPUT{\item[\textbf{Input:}]}
\algnewcommand\OUTPUT{\item[\textbf{Output:}]}

\usepackage{multirow}
\usepackage{subcaption}  

\newcommand{\mcs}{\mathcal S}
\newcommand{\mca}{\mathcal A}

\newcommand{\mE}{\mathbb{E}}

\newcommand{\mR}{\mathbb{R}}
\newcommand{\ltwo}[1]{\left\|#1\right\|_2}
\newcommand{\kl}[2]{\mathrm{KL}\left(#1 \middle\| #2\right)}
\newcommand{\lt}[1]{\left\|#1\right\|_2}
\newcommand{\prob}{\mathbb{P}}
\newcommand{\gd}{\nabla}
\newcommand{\sgd}{\widehat{\nabla}}
\newcommand{\gdt}{\nabla_\theta}
\newcommand{\gda}{\nabla_\alpha}
\newcommand{\sgdt}{\sgd_\theta}
\newcommand{\sgda}{\sgd_\alpha}
\newcommand{\defeq}{\mathrel{\mathop:}=}

\newcommand{\Ac}{\mathcal{A}}

\newcommand{\Sc}{\mathcal{S}}

\DeclareMathOperator*{\argmin}{argmin}
\DeclareMathOperator*{\argmax}{argmax} 
\newcommand{\inner}[2]{\left\langle #1, #2 \right\rangle}

\newcommand{\norm}[1]{\left\lVert #1 \right\rVert}

\newcommand{\eptt}{\mathbb{E}}

\newcommand{\tp}{{\theta}_{op}}
\newcommand{\ap}{{\alpha}_{op}}
\newcommand{\dtv}{\mathrm{d}_{TV}}
\newcommand{\bigO}[1]{\mathcal{O}\left(#1\right)}
\newcommand{\indicator}[1]{\mathbbm{1}\left\{#1\right\}}
\newcommand{\ceiling}[1]{\left\lceil #1 \right\rceil}
\newcommand{\parat}[1]{\left(#1\right)}
\newcommand{\abs}[1]{\left|#1\right|}

\newtheorem{theorem}{Theorem}
\crefname{theorem}{theorem}{theorems}
\newtheorem{lemma}{Lemma}
\crefname{lemma}{lemma}{lemmas}
\newtheorem{assumption}{Assumption}
\crefname{assumption}{assumption}{assumptions}

\crefname{remark}{remark}{remarks}
\newtheorem{definition}{Definition}
\crefname{definition}{definition}{definitions}
\newtheorem{proposition}{Proposition}
\crefname{definition}{definition}{definitions}

\crefname{corollary}{corollary}{corollaries}

\allowdisplaybreaks[4]

\usepackage{threeparttable}
\usepackage[table]{xcolor}
\renewcommand{\arraystretch}{1.4}
\usepackage{hhline}
\usepackage{makecell} 

\begin{document}

\title{When Will Generative Adversarial Imitation Learning Algorithms Attain Global Convergence}

\author{Ziwei Guan, Tengyu Xu, Yingbin Liang}
\affil{Department of Electrical and Computer Engineering, The Ohio State University}
\affil{\{guan.283, xu.3260, liang.889\}@osu.edu}

\date{}
\maketitle

\begin{abstract}
Generative adversarial imitation learning (GAIL) is a popular inverse reinforcement learning approach for jointly optimizing policy and reward from expert trajectories. A primary question about GAIL is whether applying a certain policy gradient algorithm to GAIL attains a global minimizer (i.e., yields the expert policy), for which existing understanding is very limited. Such global convergence has been shown only for the linear (or linear-type) MDP and linear (or linearizable) reward. In this paper, we study GAIL under general MDP and for nonlinear reward function classes (as long as the objective function is strongly concave with respect to the reward parameter). We characterize the global convergence with a sublinear rate for a broad range of commonly used policy gradient algorithms, all of which are implemented in an alternating manner with stochastic gradient ascent for reward update, including projected policy gradient (PPG)-GAIL, Frank-Wolfe policy gradient (FWPG)-GAIL, trust region policy optimization (TRPO)-GAIL and natural policy gradient (NPG)-GAIL. This is the first systematic theoretical study of GAIL for global convergence.
\end{abstract}
\input{introduction.tex}
\input{preliminary.tex}
\input{ppg.tex}
\input{trpo.tex}

\input{npg.tex}

\section{Conclusion}
In this paper, we study four GAIL algorithms, each of which is implemented in an alternating fashion between a popular policy gradient algorithm for the policy update and a gradient ascent for the reward update. Our focus is on investigating whether incorporation of these policy gradient algorithms to the GAIL framework will still have global convergence guarantee. We show that all these GAIL algorithms converge globally as long as the objective function is properly regularized (to be strongly concave) with respect to the reward parameter. We also anticipate that the analysis tools that we develop here will benefit the future theoretical studies of similar problems including GANs, min-max optimization, and bi-level optimization algorithms.
\section*{Acknowledgments} 
The work was supported in part by the U.S. National Science Foundation under the grants CCF-1761506, CCF-1801846, and CCF-1909291.

\newpage
\bibliography{references}
\bibliographystyle{apalike}
\newpage
\appendix
\noindent {\Large \textbf{Supplementary Materials}}
\input{appendix.tex}

\end{document}

%% file: introduction.tex
\section{Introduction}
\begin{sloppypar}
In reinforcement learning (RL), the reward function generally plays an important role to guide the design of policy optimization to attain the best long-term accumulative reward. However, a reward function may not be known in many situations, and imitation learning \cite{osa2018algorithmic} aims to find a desirable policy in such cases, which produces behaviors as close as possible to expert demonstrations.
Two popular classes of approaches for imitation learning have been developed. The first approach is behavioral cloning (BC) \cite{pomerleau1991efficient}, which directly provides a mapping strategy from the state space to the action space based on supervised learning to match expert demonstrations. The BC method often suffers from high sample complexity due to covariate shift \cite{ross2010efficient,ross2011reduction} for achieving the desired performance, which is mitigated by improved algorithms such as DAgger \cite{ross2011reduction} and Dart \cite{laskey2017iterative} that require further interaction with the expert's demonstration. The second approach is the so-called inverse reinforcement learning (IRL) \cite{russell1998learning,ng2000algorithms}, which attempts to recover the unknown reward function based on the expert's trajectories, and then find an optimal policy by using such a reward function. 

A popular IRL method has been developed in \cite{finn2016connection,ho2016generative,fu2018learning}, which leverages the connection of IRL to the training of generative adversarial networks (GANs) \cite{goodfellow2014generative}. In particular, the generative adversarial imitation learning (GAIL) framework \cite{ho2016generative} formulates a min-max optimization problem as in the GAN training. The maximization is over the reward function (which serves as a discriminator) to best distinguish between the trajectories generated by the expert and the learner, and the minimization is then over the learner's policy (which serves as a generator) to best match the expert's trajectories. Since the policy optimization in GAIL is nonconvex, its joint optimization with reward function in GAIL in general can be guaranteed to converge only to a stationary point. Such a type of result was recently established in \cite{Chen2020On}, which studied GAIL under general MDP model and reward function class, and showed that
the gradient-decent and gradient-ascent algorithm converges to a stationary point (not necessarily the global minimum).  

More recently, it has been shown that some popular policy gradient algorithms \cite{agarwal2019optimality,xu2020improving,shani2020adaptive,liu2019neural,wang2019neural} can converge to a globally optimal policy under certain policy parameterizations. Then a natural question to ask is whether such global convergence continues to hold in GAIL when these algorithms are further implemented in an alternating fashion with the reward optimization in GAIL. The global convergence does not necessarily hold in general, because the policy optimization is still over a nonconvex objective function, which can induce complicated and undesirable geometries jointly with the reward optimization as a min-max problem in GAIL. Thus, existing exploration on this topic in \cite{cai2019global,zhang2020generative}, which established global convergence for GAIL, requires restrictive conditions: (1) linear (but possibly infinite dimensional) MDP  and (2) linear reward function or linearizable reward function such as overparameterized ReLU neural networks.

This paper aims to substantially expand the aforementioned global convergence results as follows.
\begin{list}{$\bullet$}{\topsep=0.ex \leftmargin=0.15in \rightmargin=0.in \itemsep =0.02in}
\item We allow general MDP models, not necessarily linear MDP. We study nonlinear reward functions as long as the resulting objective function is strongly concave with respect to the reward parameter. This is a much bigger class than linear reward, and is satisfied easily by incorporating a strongly concave regularizer which has been commonly used in GAIL practice.

\item In addition to the projected gradient and NPG that have been studied in \cite{cai2019global,zhang2020generative} for GAIL, we also study Frank-Wolfe policy gradient, which is easier to implement than projected policy gradient, and TRPO which is widely adopted in GAIL in practice. 

\item Existing convergence characterization for GAIL assumed that the samples are either identical and independently distributed (i.i.d.) as in \cite{Chen2020On,zhang2020generative} or follows the LQR dynamics as in \cite{cai2019global}, whereas here we assume that samples follow a general Markovian distribution. 

\end{list}

\subsection{Main Contributions}

\renewcommand\arraystretch{1.5}
\begin{table*}[!t]
	\centering
	\caption{Comparison among GAIL algorithms studied in this paper}\label{tab:results}
	\vspace{0.1cm}
\begin{threeparttable} 
	\begin{tabular}{|c|c|c|}
		\hline
		Algorithms & Convergence rate & Total Complexity\tnote{1,2}$\quad$\\
		\hhline{|---|} 
		PPG-GAIL  & $\bigO{\frac{1}{(1-\gamma)^3\sqrt{T}}}$ & $\tilde{\mathcal{O}}\parat{\frac{1}{\epsilon^4}}$ \\ 
		\cline{1-3}
		FWPG-GAIL & $\bigO{\frac{1}{(1-\gamma)^3\sqrt{T}}}$ &  $\tilde{\mathcal{O}}(\frac{1}{\epsilon^4})$ \\ 
		\cline{1-3}
		TRPO-GAIL (unregularized)&$\bigO{\frac{1}{(1-\gamma)^2\sqrt{T}}}$& $\tilde{\mathcal{O}}(\frac{1}{\epsilon^3})$\\
		\cline{1-3}
		TRPO-GAIL (regularized) &$\tilde{\mathcal{O}}\parat{\frac{1}{(1-\gamma)^3T}}$ & $\tilde{\mathcal{O}}(\frac{1}{\epsilon^2})$ \\
		\cline{1-3} 
		NPG-GAIL & $\bigO{\frac{1}{(1-\gamma)^2\sqrt{T}}}$&    $\tilde{\mathcal{O}}(\frac{1}{\epsilon^4})$\\
		\cline{1-3}
	\end{tabular} 
\begin{tablenotes}
	\item[1] Total complexity refers to the total number of samples needed to achieve an $\epsilon$-accurate globally optimal point.
	\item[2] $\tilde{\mathcal{O}}(\cdot)$ does not include the logarithmic terms. 
\end{tablenotes}
\end{threeparttable} 
\end{table*}
In this paper, we establish the first {\em global} convergence guarantee for GAIL under the general MDP model and the nonlinear reward function class (as long as the objective function is strongly concave with respect to the reward parameter). We provide the convergence rate for three major types of algorithms, all of which alternate between gradient ascent (for reward update) and policy gradient descent (for policy update), respectively being (a) projected policy gradient (PPG)-GAIL and Frank-Wolfe policy gradient (FWPG)-GAIL (with direct policy parameterization); (b) trust region policy optimization (TRPO)-GAIL (with direct policy parameterization); and (c) natural policy gradient (NPG)-GAIL (with general non-linear policy parameterization). We show that all these alternating algorithms converge to the {\em global} minimum with a sublinear rate. We summarize our results on the convergence performance of the GAIL algorithms in \Cref{tab:results}. Comparing among these algorithms indicates that TRPO-GAIL with regularized MDP achieves the best convergence rate, and TRPO-GAIL with regularized and unregularized MDP outperform the other algorithms in terms of the overall sample complexity.

Technically, the global convergence guarantee for GAIL does not follow from the existing min-max optimization theory. In fact, the GAIL problem here falls into nonconvex-strongly-concave min-max optimization framework, for which existing optimization theory does not provide the global convergence in general. Thus, our establishment of global convergence for GAIL develops several new properties specially for GAIL. Furthermore, in contrast to conventional min-max optimization, which is under i.i.d.\ sampling by certain static distribution, GAIL is under Markovian sampling by time-varying distributions due to the policy update. Thus, the convergence analysis for GAIL is more challenging than that for min-max optimization.

\subsection{Related Work}\label{sec:relatedwork}
Due to the significant growth of studies in imitation learning, this section focuses only on those studies that are highly relevant to the theoretical analysis of the convergence for GAIL algorithms.

{\bf Theory for IRL via adversarial training:} The idea of generative adversarial training \cite{goodfellow2014generative} has motivated a popular approach for IRL problems \cite{finn2016connection,ho2016generative,fu2018learning}. Among these studies, GAIL \cite{ho2016generative} formulated a min-max problem for jointly optimizing the reward and policy, where reward and policy serve analogous roles as the discriminator and the generator in GANs. Naturally, such an approach has been explored via the divergence  minimization perspective in \cite{ke2019imitation,ghasemipour2019divergence}, by leveraging GAN training \cite{nowozin2016fgan}. Moreover, the generalization performance and sample complexity have been studied for the setting where the expert's demonstrations include only the states but no actions. 

Most relevant to our study is the recent studies \cite{cai2019global,Chen2020On,zhang2020generative} on the convergence rate for the algorithms developed for GAIL. Among these studies, \cite{Chen2020On} studied GAIL under the general MDP model and the reward function class, and showed that the gradient-decent and gradient-ascent algorithm converges to a stationary point (not necessarily the global minimum). \cite{cai2019global,zhang2020generative} provided the global convergence result. More specifically, \cite{cai2019global} studied GAIL under linear quadratic regulator (LQR) dynamics and the linear reward function class, and showed that the alternating gradient algorithm converges to the unique saddle point. \cite{zhang2020generative} studied GAIL under a type of linear but infinite dimensional MDP and with overparameterized neural networks for parameterizing the policy and reward function, and showed that the alternating algorithm between gradient-ascent (for reward update) and NPG (for policy update) converges to the neighborhood of a global optimal point, where the representation power of neural networks determines the convergence error. Our study here establishes global convergence for GAIL for general MDP and the nonlinear reward function class.

{\bf Difference from conventional min-max problems:} Although the GAIL framework is formulated as a min-max optimization problem, the stochastic algorithms that we use for solving such a problem have the following major differences from the conventional min-max optimization problem. First, since these algorithms continuously update the policy, the samples that are used for iterations are sampled by time-varying policies; whereas the conventional min-max problem typically has a fixed sampling distribution. Second, since the samples are obtained following an MDP process, the samples are distributed with correlation rather than in the i.i.d.\ manner as in the conventional optimization. These two differences cause the convergence analysis to be more complicated for GAIL than the conventional min-max problem. Furthermore, the min-max problem that we encounter here for GAIL is nonconvex-strongly-concave, for which the conventional min-max optimization \cite{nouiehed2019solving,lin2020near} has been shown to converge only to a stationary point, whereas this paper exploits further properties in GAIL and establishes the global convergence guarantee.

{\bf Connection to policy gradient algorithms:} In the GAIL framework, the policy optimization is jointly performed with the reward optimization via a min-max optimization. Thus, the variation of the reward function during the algorithm execution continuously change the objective function for the policy optimization. Hence, even if the policy gradient algorithms (running for a fixed objective function) converge globally, for example, PPG \cite{agarwal2019optimality}, NPG \cite{agarwal2019optimality}, and TRPO \cite{shani2020adaptive}, the global convergence is generally not guaranteed if these algorithms are executed in an alternating fashion with reward iterations. Two special cases have been shown to retain such global convergence, namely, LQR model shown in \cite{cai2019global} and overparameterized neural networks for a linear type MDP\cite{zhang2020generative}. This paper significantly expands such a set of cases by establishing the global convergence guarantee for more general MDP and reward class and a broader range of algorithms.
\end{sloppypar}

%% file: preliminary.tex
\section{Problem Formulation and Preliminaries}\label{sec:preliminaries}
\subsection{Markov Decision Process}
The imitation learning framework that we study is based on the Markov decision process (MDP) denoted by $(\mcs, \mca, \mathsf{P},r,\gamma)$. We assume that both the state space $\mcs\subset \mR^d$ and the action space $\mca$ are finite, and use $s\in \mcs$ and $a\in \mca$ to denote a state and an action, respectively. A policy $\pi$ describes the probability to take an action $a\in \mca$ at each state $s \in \mcs$ in terms of the conditional probability $\pi(a|s)$.  Then the system moves to a next state $s'\in \mcs$ governed by the probability transition kernel $\mathsf{P}(s'|s,a)$, and receives a reward $r_t=r(s, a)$, which is assumed to be bounded by  $R_{\max}$.


Suppose the initial state takes a distribution $\zeta$. For a given policy $\pi$ and a reward function $r$, we define the average value function as: 
\begin{align*}
\resizebox{0.98\hsize}{!}{$V(\pi,r)=\mE\big[\sum_{t=0}^{\infty}\gamma^t r(s_t,a_t)\big|s_0\sim \zeta, a_t\sim \pi(a_t|s_t),s_{t+1}\sim \mathsf{P}(s_{t+1}|s_t,a_t)\big]=\frac{1}{1-\gamma}\eptt_{(s,a)\sim \nu_{\pi}(s,a)}[r(s,a)]$},
\end{align*}
 where $\gamma\in(0,1)$ is a discount factor and $\nu_\pi(s, a) \defeq (1-\gamma)\sum_{t=0}^\infty\gamma^t \prob(s_t =s, a_t = a)$ is the  state-action visitation distribution. 
It has been shown in \cite{konda2002actor} that $\nu_\pi(s,a)$ is the stationary distribution of the Markov chain with the transition kernel $\tilde{\mathsf{P}}(\cdot|s,a) = (1-\gamma)\zeta(\cdot) + \gamma \mathsf{P}(\cdot|s,a)$ and policy $\pi$ if the Markov chain is ergodic. Thus $\tilde{\mathsf{P}}$ is used in sampling for estimating the value function.


\subsection{Generative Adversarial Imitation Learning (GAIL)}\label{sec:gail}
For imitation learning, in which the reward function is not known, GAIL \cite{ho2016generative} is a framework to jointly learn the reward function and optimize the policy. We parameterize the reward function by $\alpha \in \Lambda\subset\mathbb{R}^q$, which takes the form $r_\alpha(s,a)$ at the state-action pair $(s,a)$. We assume that $\Lambda$ is a bounded closed set, i.e., $\norm{\alpha_1-\alpha_2}_2 \le C_\alpha$, $\forall \alpha_1, \alpha_2\in\Lambda$. 


We let $\pi_E$ represent the expert policy, and let the learner's policy be parameterized by $\theta \in \Theta$ and be denoted as $\pi_\theta$. In this paper, we consider two types of parameterization for the learner's policy. The first is the direct parameterization, where $\theta =\{\theta_{s,a}, s\in \mcs, a\in \mca\}$, and $\pi_{\theta}(a|s)=\theta_{s,a}$ where $\theta \in \Theta_p:=\{\theta: \theta_{s,a} \ge 0, \sum_{a\in \mca} \theta_{s,a}=1, \; \text{ for all }\; s\in \mcs, a\in \mca \}$. The second is the general nonlinear policy class, which satisfies certain smoothness conditions as given in \Cref{assp:generalpolicyparameterization}.

The GAIL framework is formulated as the following min-max optimization problem.
\begin{align}\label{eq:minmax}
\min_{\theta \in \Theta}\max_{\alpha \in \Lambda} F(\theta,\alpha):= V(\pi_E,r_\alpha)-V(\pi_\theta,r_\alpha)-\psi(\alpha),
\end{align}
where the objective function is given by the discrepancy of the accumulated rewards between the expert's and learner's policies, regularized by a function $\psi(\alpha)$ of the reward parameter. Thus, the maximization in \cref{eq:minmax} aims to find the reward function that best distinguishes between the expert's and the learner's policies and the minimization aims to find the learner's policy that matches the expert's policy as close as possible. Such a formulation is analogous to the GANs, with the reward serving as a discriminator and the policy serving as a generator. 

\begin{algorithm}
	\caption{Nested-loop GAIL framework}\label{alg:nestedloopIL}
	\begin{algorithmic}[1]
		\STATE \textbf{Input:} Outer loop length $T$, inner loop length $K$, stepsize $\eta$, $\beta$
		\FOR {$t= 0, 1, ..., T-1$}
		\STATE Randomly pick $\alpha_0^t \in \Lambda$
		\FOR {$k= 0, 1, ..., K-1$}
		\STATE Query a length-$B$ trajectory $(s^E_i,a^E_i) \sim \mathsf{\tilde{P}}^{\pi_E}$ and a length-$B$ mini-batch $(s_i^\theta, a_i^\theta)\sim \mathsf{\tilde{P}}^{\pi_\theta}$\footnotemark
		\STATE $\sgda F(\theta, \alpha) = \frac{1}{(1-\gamma)B} \sum_{i=0}^{B-1}\left[ \gda r_{\alpha} (s_i^E, a_i^E) - \gda r_{\alpha} (s_i^{\theta}, a_i^{\theta})\right] - \gda \psi(\alpha)$
		\STATE $\alpha_{k+1}^t = P_{\Lambda}\parat{\alpha_{k}^t + \beta\sgda F(\theta_t, \alpha_{k}^t)}$
		\ENDFOR
		\STATE $\alpha_t = \alpha_{K}^t$
		\STATE $\theta_{t+1} =$ {\bf Options}: {\bf PPG} in \cref{eq:ppgupdate}; {\bf FWPG} in \cref{eq:fwupdate}; {\bf TRPO} in \cref{eq:trpo}; {\bf NPG}
in \cref{eq:npg}		
		\ENDFOR
	\end{algorithmic}
\end{algorithm}
\footnotetext{The samples are obtained over a single trajectory path for the entire algorithm execution.}

In this paper, we study four GAIL algorithms, all of which follow the nested-loop framework described in \Cref{alg:nestedloopIL}. Namely, at each time step $t$ (associated with one outer loop), there is an entire inner loop updates of the reward parameter $\alpha_t$ to a certain accuracy and one update step of the policy parameter $\theta_t$. Specifically, $\alpha_t$ is updated by the stochastic projected gradient ascent given by
\begin{align*}
\alpha^{k+1}_{t}=P_{\Lambda}\parat{\alpha^k_{t}+\beta \widehat{\nabla}_\alpha F(\theta_t,\alpha^k_{t})},
\end{align*} 
where the gradient estimator $\widehat{\nabla}_\alpha F(\theta_t,\alpha^k_{t})$ is obtained via a Markovian sample trajectory. Then the policy parameter $\theta_t$ is updated for one step, determined by any of the four policy gradient algorithms, namely, PPG in \cref{eq:ppgupdate}, FWPG in \cref{eq:fwupdate}, TRPO in \cref{eq:trpo} and NPG in \cref{eq:npg}.  

\subsection{Technical Preliminaries}

For the GAIL problem in \cref{eq:minmax} to be well posed, we assume that $\max_{\alpha\in\Lambda} F(\theta, \alpha)$ exists for any $\theta\in \Theta$, and define the marginal-maximum function 
of $F(\theta, \alpha)$
\begin{align}\label{eq:gh}
g(\theta): = \max_{\alpha\in\Lambda} F(\theta, \alpha).
\end{align}
 We further define the corresponding optimizer $\ap(\theta)  := \argmax_{\alpha \in \Lambda} F(\theta, \alpha)$. If there exists more than one optimizer, $\ap(\theta)$ denotes the elements of the corresponding optimizer set. 
\begin{definition}
Let $\theta^* = \argmin_{\theta \in\Theta} g(\theta)$. The output $\bar{\theta}$ of an algorithm is said to attain an $\epsilon$-global convergence if $g(\bar{\theta}) - g(\theta^*) \leq \epsilon$ holds for a prescribed accuracy $\epsilon \in (0,1)$.
\end{definition}
As remarked in \cite{zhang2020generative}, $\epsilon$-global convergence further implies 
\begin{align*}
\max_{\alpha \in \Lambda}[V(\pi_E,r_\alpha)-V(\pi_{\bar{\theta}},r_\alpha)] \leq \max_{\alpha \in \Lambda} \psi(\alpha)+\epsilon.
\end{align*}
Hence, as long as $\psi(\alpha)$ is chosen properly (for example, with a small regularization coefficient), $\pi_{\bar{\theta}}$ is guaranteed to be sufficiently close to the expert policy.


In this paper, we make the following standard assumptions for our analysis.
\begin{assumption}\label{assp:regularizer}
The regularizer function $\psi(\alpha)$ is differentiable with gradient Lipschitz constant $L_\psi$.
\end{assumption}
\Cref{assp:regularizer} captures the property for designing a regularizer and can be easily attained.

\begin{assumption}\label{assp:objectivefunction}
For any given $\theta$, the objective function $F(\theta, \alpha)$ in \cref{eq:minmax} is $\mu$-strongly concave on $\alpha$.
\end{assumption}
\Cref{assp:objectivefunction} includes the linear function class as a special case. In practice, a strongly convex regularizer $\psi(\alpha)$ is often used to guarantee the strong concavity of $F(\theta, \alpha)$.
\begin{assumption}[Ergodicity]\label{aspt:ergodic}
	For any policy parameter $\theta\in \Theta$, consider the MDP with policy $\pi_\theta$ and transition kernel $\mathsf{P}(\cdot|s,a)$ or $\tilde{\mathsf{P}}(\cdot|s,a) = \gamma\mathsf{P}(\cdot|s,a) + (1-\gamma)\zeta(\cdot)$. There exist constants $C_M>0$ and $0<\rho<1$ such that $\forall t\ge 0$, 	
	\begin{equation*}
	\sup_{s\in \Sc}\dtv\left(\mathbb{P}(s_t\in\cdot|s_0 =s), \chi_{\theta}\right) \le C_M \rho^{t},
	\end{equation*}
	where $\chi_{\theta}$ is the stationary distribution of the given transition kernel $\mathsf{P}(\cdot|s,a)$ or $\tilde{\mathsf{P}}(\cdot|s,a)$ under policy $\pi_\theta$ and $\dtv\left(\cdot, \cdot\right)$ is the total variation distance. 
\end{assumption}
\Cref{aspt:ergodic} holds for any time-homogeneous Markov chain with finite state space or any uniformly ergodic Markov chain with general state space.
\begin{assumption}\label{assp:rewardassumption}	
	The reward parameterization satisfies the following requirements:
	\begin{list}{$\bullet$}{\topsep=-0.1in \leftmargin=0.3in \rightmargin=0.in \itemsep =-0.1in}
		\item[(1)] Bounded gradient: $\exists C_r\in\mathbb{R}$ such that $\forall \alpha\in\Lambda$, $\norm{\gda r_\alpha}_{\infty,2} := \sqrt{\sum_{i=1}^q \norm{\frac{\partial  r_\alpha}{\partial \alpha_i}}_\infty^2}\le C_r$.
		\item[(2)] Gradient Lipschitz: $\exists L_r\in\mathbb{R}$, such that $\forall s\in\Sc, a\in \Ac$ and $\forall \alpha_1, \alpha_2 \in \Lambda$,
		$$\lt{\gda r_{\alpha_1} (s, a) - \gda r_{\alpha_2} (s, a)} \le L_r \lt{\alpha_1 - \alpha_2}.$$
	\end{list}
\end{assumption}

We next provide the following Lipschitz properties, which are vital for the analysis of convergence, and were often taken as assumptions in the literature of min-max optimization \cite{jin2019local,nouiehed2019solving}. 
\begin{proposition}\label{lemma:lipschitzcondtion}
	Suppose Assumptions \ref{assp:regularizer}, \ref{aspt:ergodic} and \ref{assp:rewardassumption} hold. Then the GAIL min-max problem in \cref{eq:minmax} with direct parameterization satisfies the following Lipschitz conditions:  $\forall \theta_1, \theta_2\in \Theta$ and $\forall \alpha_1, \alpha_2 \in \Lambda$,
	\begin{align*}
	\norm{\gdt F(\theta_1, \alpha_1) - \gdt F(\theta_2, \alpha_2)} &\le L_{11} \norm{\theta_1 - \theta_2} + L_{12} \norm{\alpha_1 -\alpha_2},\\
	\norm{\gda F(\theta_1, \alpha_1) - \gda F(\theta_2, \alpha_2)} &\le L_{21} \norm{\theta_1 - \theta_2} + L_{22} \norm{\alpha_1 -\alpha_2},
	\end{align*}
	where $L_{11} =\frac{2\sqrt{2}|\Ac|C_rC_\alpha}{(1-\gamma)^2}(1 + \ceiling{\log_{\rho} C_M^{-1}} + (1-\rho)^{-1})$, $L_{12} = \frac{\sqrt{|\Ac|}C_r}{(1-\gamma)^2}$, $L_{21} = \frac{C_r\sqrt{|\Ac|}}{1-\gamma}(1 + \lceil\log_{\rho} C_M^{-1}\rceil + (1-\rho)^{-1})$, and $L_{22} = \frac{2\sqrt{q}L_r}{1-\gamma} + L_\psi$. Furthermore, if $\theta_1 = \theta_2$, the above second bound holds with a general parameterization for the policy.
\end{proposition}



%% file: ppg.tex
\section{Global Convergence of GAIL Algorithms}
In this section, we provide the global convergence guarantee for four GAIL algorithms.
\subsection{PPG-GAIL and FWPG-GAIL Algorithms}
In this section, we study the PPG-GAIL and FWPG-GAIL algorithms, both of which take the general framework in \Cref{alg:nestedloopIL}, and update the policy parameter $\theta$ respectively based on projected policy gradient (PPG) and Frank-Wolfe policy gradient (FWPG). 

We take the direct parameterization for the policy. At each time $t$ of the outer loop, both PPG-GAIL and FWPG-GAIL first estimate the stochastic policy gradient by drawing a minibatch sample trajectory with length $b$ as $(s_i,a_i)\sim \mathsf{\tilde{P}}^{\pi_{\theta_t}}$ as follows.
\begin{equation}
\sgdt F(\theta_t, \alpha_t) (s,a) = -\frac{\hat{Q}(s, a)}{b(1-\gamma)}\sum_{i=0}^{b-1}\indicator{s_i= s}, \label{eq:policygradient}
\end{equation}
for all $s\in\Sc, a\in \Ac$, where $\hat{Q}(s, a)$ applies EstQ in \cite{zhang2019global} (see \Cref{app:alg}) with the reward function $r_{\alpha_t}(s,a)$. Then, PPG-GAIL updates $\theta_t$ as
\begin{align}\label{eq:ppgupdate}
\theta_{t+1}=P_{\Theta_p}\parat{\theta_{t}-\eta \sgdt F(\theta_t,\alpha_t)},
\end{align} 
where $\Theta_p$ is the probability simplex defined in \Cref{sec:gail}.

Differently from PPG-GAIL, FWPG-GAIL updates $\theta_t$ based on the Frank-Wolfe gradient as given by
\begin{align}\label{eq:fwupdate}
&\hat{v}_t = \argmax_{\theta\in\Theta_p} \inner{\theta}{-\sgdt F(\theta_{t}, \alpha_{t})}, \qquad
\theta_{t+1} = \theta_t + \eta \left(\hat{v}_t - \theta_t\right).
\end{align} 



To analyze the convergence, we first define the gradient dominance property.
\begin{definition}\label{def:pl-likecondition}
A function $f(\theta)$ satisfies the gradient dominance property, if there exists a positive $C$, such that $f(\theta) - f(\theta^*) \le C \max_{\bar{\theta}\in \Theta} \inner{\theta -\bar{\theta}}{\gd_\theta f(\theta)}$ for any given $\theta\in \Theta$, where $\theta^*:=\argmin_{\theta\in \Theta}f(\theta)$.
\end{definition} 
The following proposition facilitates to prove global convergence for PPG-GAIL and FWPG-GAIL.
\begin{proposition}\label{prop:PL-like}
The function $g(\theta)$ given in \cref{eq:gh} satisfies the gradient dominance property.
\end{proposition}


The following theorem characterizes the global convergence of PPG-GAIL.
\begin{theorem}\label{thm:ppg}
Suppose \Cref{assp:regularizer,assp:objectivefunction,aspt:ergodic,assp:rewardassumption} hold. Consider PPG-GAIL 
with the $\theta$-update stepsize $\eta =  \left(L_{11} +  \frac{L_{12}L_{21}}{\mu}\right)^{-1}$ and the $\alpha$-update stepsize $\beta = \frac{\mu}{4L_{22}^2}$, where $L_{11}$, $L_{12}$, $L_{21}$ and $L_{22}$ are given in \Cref{lemma:lipschitzcondtion}. Then we have 
	\begin{align}\label{eq:ppgconv}
	\frac{1}{T}&\sum_{t=0}^{T-1} \eptt\left[ g(\theta_t)\right] - g(\theta^*)  \le  \resizebox{0.65\hsize}{!}{$\bigO{\frac{1}{(1-\gamma)^{3} \sqrt{T}}} + \bigO{{e^{-(1-\gamma)^2K}}} + \bigO{\frac{1}{(1-\gamma)^3 \sqrt{B}}}+ \bigO{\frac{1}{(1-\gamma)^3\sqrt{b}}}$}. 
	\end{align}
\end{theorem}

\Cref{thm:ppg} implies that if we set $T = \bigO{\frac{1}{\epsilon^2}}$, $K = \bigO{\log(\frac{1}{\epsilon})}$, $B =\bigO{\frac{1}{\epsilon^2}}$ and $b =\bigO{\frac{1}{\epsilon^2}}$, then PPG-GAIL converges to an $\epsilon$-accurate {\em globally} optimal value with an overall sample complexity $T(KB + b) = \tilde{\mathcal{O}}\parat{\frac{1}{\epsilon^4}}$. Due to the Markovian sampling for updating both the reward and policy parameters $\alpha$ and $\theta$, our analysis bounds the two corresponding bias error terms by $\mathcal{O}( \frac{1}{\sqrt{B}})$ and $\mathcal{O}(\frac{1}{\sqrt{b}})$ as shown in \cref{eq:ppgconv}. Hence, the choices for the mini-batch sizes $B$ and $b$ trade off between the convergence error and the computational complexity. To achieve a given accuracy $\epsilon$, the tradeoff yields the overall complexity of $ \tilde{\mathcal{O}}\parat{\frac{1}{\epsilon^4}}$. We also note that the result here provides the first convergence rate for projected stochastic gradient with non-i.i.d. sampling.


We next provide the following theorem, which characterizes the global convergence of FWPG-GAIL.
\begin{theorem}\label{thm:fwalg}
	Suppose \Cref{assp:regularizer,assp:objectivefunction,aspt:ergodic,assp:rewardassumption} hold. Consider FWPG-GAIL with the $\theta$-update stepsize $\eta = \frac{1-\gamma}{\sqrt{T}}$ and $\alpha$-update stepsize $\beta = \frac{\mu}{4L_{22}^2}$, where $L_{22}$ is given in \Cref{lemma:lipschitzcondtion}.
	Then we have 
	\begin{align*}
	\frac{1}{T} \sum_{t=0}^{T-1} &\eptt\left[g(\theta_t)\right] - g(\theta^*) \le \bigO{\frac{1}{(1-\gamma)^{3} \sqrt{T}}} + \bigO{{e^{-(1-\gamma)^2K}}}+ \bigO{\frac{1}{(1-\gamma)^3 \sqrt{B}}} 
	+ \bigO{\frac{1}{(1-\gamma)^3 \sqrt{b}}}.
	\end{align*}
\end{theorem}
\Cref{thm:fwalg} implies that if we let $T = \bigO{\frac{1}{\epsilon^2}}$, $K = \bigO{\log(\frac{1}{\epsilon})}$, $B=\bigO{\frac{1}{\epsilon^2}}$ and $b=\bigO{\frac{1}{\epsilon^2}}$, then FWPG-GAIL converges to an $\epsilon$-accurate {\em globally} optimal value with overall sample complexity $T(KB + b) = \tilde{\mathcal{O}}\parat{\frac{1}{\epsilon^4}}$, which is the same as that of PPG-GAIL. The analysis of FWPG-GAIL also needs to bound the two bias terms due to the Markovian sampling for updating the reward and policy parameters. This is the first analysis that provides the convergence rate for stochastic Frank-Wolfe gradient with non-i.i.d. sampling.

%% file: trpo.tex
\vspace{-2mm}
\subsection{TRPO-GAIL Algorithm}
\vspace{-2mm}
In this section, we study the TRPO-GAIL algorithm, which takes the general framework in \Cref{alg:nestedloopIL} and updates the policy parameter $\theta$ based on TRPO under $\lambda$-regularized MDP. At each time $t$ of the outer loop, TRPO-GAIL adopts the update rule in \cite{shani2020adaptive} for updating $\theta_t$ as follows:
\begin{align*}
	\pi_{\theta_{t+1}}(\cdot|s) \in \argmin_{\pi \in \Delta_{\Ac}} \inner{-\hat{Q}^{\pi_{\theta_t}}_{\lambda,\alpha_t}(s, \cdot) + \lambda \nabla \omega (\pi_{\theta_t}(\cdot | s))}{\pi - \pi_{\theta_t}(\cdot|s)} + \eta_t^{-1} B_{\omega} (\pi, \pi_{\theta_t}(\cdot |s)),
\end{align*}
where $\hat{Q}^{\pi_{\theta_t}}_{\lambda, \alpha_t}$ denotes the estimation of the Q-function based on EstQ \cite{zhang2019global} (see \Cref{app:alg}), the regularized reward $r_{\lambda,\alpha_t}(s, a) \defeq r_{\alpha_t}(s,a) + \lambda\omega(\pi_{\theta}(\cdot|s))$, the negative entropy function $\omega(\pi(\cdot|s))\defeq \sum_{a\in\Ac}\pi(\cdot|s)\log \pi(\cdot|s) +\log{|\Ac|}$, and the Bregman distance $B_{\omega}(x, y) \defeq \omega(x) - \omega(y) - \inner{\nabla \omega(y)}{x-y}$ associated with $\omega(x)$, which is the KL-divergence here. We consider the direct parameterization for the policy, and hence the update for the policy parameter $\theta$ can be analytically computed \cite{shani2020adaptive} as follows. For each $(s,a) \in \Sc\times\Ac$,
\begin{align}\label{eq:trpo}
\theta_{t+1}(s,a) = \frac{ \theta_t(s, a) \exp\left(\eta_t (\hat{Q}^{\pi_{\theta_t}}_{\lambda, \alpha_t}(s,a) - \lambda \log \theta_t(s,a) )\right)}{\sum_{a' \in \Ac} \theta_t(s, a') \exp\left(\eta_t (\hat{Q}^{\pi_{\theta_t}}_{\lambda, \alpha_t}(s,a') - \lambda \log \theta_t(s,a') )\right)}.
\end{align}

The following theorem provides the global convergence of TRPO-GAIL under the unregularized MDP, where $\lambda = 0$.
\begin{theorem}\label{thm:trpo1} 
		Suppose \Cref{assp:regularizer,assp:objectivefunction,aspt:ergodic,assp:rewardassumption} hold. Consider unregularized TRPO-GAIL ($\lambda = 0$) with 
		$\theta$-update stepsize $\eta_t = \frac{1-\gamma}{\sqrt{T}}$ and $\alpha$-update stepsize $\beta = \frac{\mu}{4L_{22}^2}$, where $L_{22}$ is given in \Cref{lemma:lipschitzcondtion}. Then we have, 
		\begin{align*}
		\frac{1}{T} \sum_{t=0}^{T-1} \eptt\left[g(\theta_t)\right] - g(\theta^*) &\le \bigO{\frac{1}{(1-\gamma)^2 \sqrt{T}}} + \bigO{e^{-(1-\gamma)^2K}}+ \bigO{\frac{1}{(1-\gamma)^4 B}}.
		\end{align*}
\end{theorem}		
		
We further consider the regularized MDP, where $\lambda>0$.
\begin{theorem}\label{thm:trpo2}		
	Suppose \Cref{assp:regularizer,assp:objectivefunction,aspt:ergodic,assp:rewardassumption} hold. Consider regularized TRPO-GAIL ($\lambda >0$) with $\theta$-update stepsize $\eta_t = \frac{1}{\lambda(t+2)}$ and $\alpha$-update stepsize $\beta =\frac{\mu}{4L_{22}^2}$, where $L_{22}$ is given in \Cref{lemma:lipschitzcondtion}. Then we have, 
		\begin{align*}
		\frac{1}{T} \sum_{t=0}^{T-1} \eptt\left[g(\theta_t)\right]  -g(\theta^*) &\le \tilde{\mathcal{O}}\parat{{\frac{1}{(1-\gamma)^3 T}}} + \bigO{e^{-(1-\gamma)^2K}}+ \bigO{\frac{1}{(1-\gamma)^4 B}}.
		\end{align*}
\end{theorem}
\Cref{thm:trpo1} indicates that if we set $T = \bigO{\frac{1}{\epsilon^2}}$, $K =\bigO{\log(\frac{1}{\epsilon})}$ and $B = \bigO{\frac{1}{\epsilon}}$, then TRPO-GAIL with unregularized MDP converges to an $\epsilon$-accurate {\em globally} optimal value with a total sample complexity $TKB =  \tilde{\mathcal{O}}\parat{\frac{1}{\epsilon^3}}$. \Cref{thm:trpo2} indicates that if we let $T=\tilde{\mathcal{O}}(\frac{1}{\epsilon})$, $K = \bigO{\log(\frac{1}{\epsilon})}$, and $B = \bigO{\frac{1}{\epsilon}}$, then TRPO-GAIL with regularized MDP converges to an $\epsilon$-accurate {\em globally} optimal value with an overall sample complexity $TKB = \tilde{\mathcal{{O}}}\parat{\frac{1}{\epsilon^2}}$. The regularized MDP changes the objective function with $\lambda$-regularized perturbation and yields orderwisely better sample complexity. Moreover, the sample complexity here is with respect to the convergence in expectation, which improves that in high-probability convergence in \cite{shani2020adaptive} by a factor of $\tilde{\mathcal{O}}\parat{\frac{1}{\epsilon}}$.


%% file: npg.tex
\subsection{NPG-GAIL Algorithm}

In this section, we study the NPG-GAIL algorithm, which takes the general framework in \Cref{alg:nestedloopIL} and updates the policy parameter $\theta$ based on natural policy gradient (NPG).

We consider the general nonlinear parameterization for the policy, so that the state space may not be finite and for example can be $\mathbb{R}^d$. At each time $t$ of the outer loop, NPG-GAIL ideally should update $\theta_t$ via a regularized natural gradient $-(F(\theta_t)+ \lambda I)^{-1}\gdt V(\pi_{\theta_t}, r_{\alpha_t})$, where $F(\theta) = \eptt_{(s,a)\sim\nu_{\pi_{\theta}}}\left[\gdt\log(\pi_{\theta}(a|s)) \gdt\log(\pi_{\theta}(a|s))^\top\right]$ is the Fisher-information matrix, and $\lambda$ is the regularization coefficient for avoiding singularity. In practice, we estimate such a natural gradient via solving the problem
$\min_{w\in R^d}\eptt_{(s,a)\sim\nu_{\pi_\theta}} \left[\gdt \log(\pi_{\theta}(a|s))^\top w  - A^{\pi_\theta}_\alpha(s,a)\right]^2$  using the mini-batch linear stochastic approximation (SA) algorithm over a Markovian sampled trajectory, where $A^{\pi_\theta}_\alpha(s,a) \defeq Q^{\pi_\theta}_\alpha(s,a) - V^{\pi_\theta}_\alpha(s)$ is the advance function under reward $r_\alpha$. More details are provided in \Cref{alg:npggail} in \Cref{app:alg}. Suppose such an algorithm provides an output $w_t$. Then the policy parameter is updated as
\begin{align}\label{eq:npg}
\theta_{t+1}=\theta_t - \eta w_t.
\end{align}

Since we take the general nonlinear parameterization for the policy, we make the following assumptions for the policy parameterization, which are standard in the literature \cite{kumar2019sample,zhang2019global,agarwal2019optimality,xu2020reanalysis}.
\begin{assumption}\label{assp:generalpolicyparameterization}
	For any $\theta, \theta^\prime\in \Theta$, and any state-action pair $(s,a)\in \Sc\times\Ac$, there exist positive constants $L_\pi$, $L_\phi$, $C_\phi$ and $C_\pi$, such that the following bounds hold:
	\begin{list}{$\bullet$}{\topsep=0.ex \leftmargin=0.3in \rightmargin=0.in \itemsep =-0.0in}
		\item[(1)] $\lt{\gdt \log(\pi_{\theta}(a|s)) - \gdt \log(\pi_{\theta^\prime}(a|s))} \le L_{\phi}\lt{\theta - \theta^\prime}$,
		\item[(2)] $\lt{\gdt \log(\pi_{\theta}(a|s))} \le C_{\phi}$,
		\item[(3)] $\norm{\pi_{\theta}(\cdot|s) - \pi_{\theta^\prime} (\cdot|s)}_{TV} \le C_\pi \lt{\theta - \theta^\prime}$, where
		$\norm{\cdot}_{TV}$ denotes the total-variation norm.
	\end{list}
\end{assumption}

Next, we provide the following theorem, which characterizes the global convergence of NPG-GAIL.
\begin{theorem}\label{thm:npg}
 Suppose \Cref{assp:regularizer,assp:objectivefunction,aspt:ergodic,assp:rewardassumption,assp:generalpolicyparameterization} hold. Consider NPG-GAIL 
with $\theta$-update stepsize $\eta= \frac{1-\gamma}{\sqrt{T}}$, $\alpha$-update stepsize $\beta = \frac{\mu}{4L_{22}^2}$, and the SA-update stepsize $\beta_W = \frac{\lambda_P}{4(C_\phi^2 +\lambda)^2}$, where $L_{22}$ is given in \Cref{lemma:lipschitzcondtion}. Then we have
\begin{align*}
\frac{1}{T} \sum_{t=0}^{T-1} \eptt\left[g(\theta_t)\right] - g(\theta^*) &\le \bigO{\frac{1}{(1-\gamma)^2 \sqrt{T}}} + \bigO{e^{-(1-\gamma)^2K}}+ \bigO{\frac{1}{(1-\gamma)^4 B}} + \bigO{e^{-T_c}} \\
&\quad+\bigO{\frac{\zeta^\prime}{(1-\gamma)^{3/2}}}+\bigO{\frac{\lambda}{1-\gamma}}+ \bigO{\frac{1}{(1-\gamma)^2\sqrt{M}}}.
\end{align*}
	where $\zeta^\prime = \max_{\theta\in\Theta, \alpha\in\Lambda}\min_{w\in R^d}\sqrt{\eptt_{\nu_{\pi_\theta}} \left[\gdt \log(\pi_{\pi_\theta}(a|s))^\top w  - A^{\pi_\theta}_\alpha(s,a)\right]^2}$ and $T_c$ and $M$ are defined in \Cref{alg:npggail} in \Cref{app:alg}. 
\end{theorem}
\Cref{thm:npg} indicates that if we let $T =\bigO{\frac{1}{\epsilon^2}}$, $K = \bigO{\log(\frac{1}{\epsilon})}$, $B = \bigO{\frac{1}{\epsilon}}$, $T_c= \bigO{\log(\frac{1}{\epsilon})}$, $\lambda = \bigO{\zeta'}$ and $M = \bigO{\frac{1}{\epsilon^2}}$, then NPG-GAIL converges to an $(\epsilon+\bigO{\zeta'})$-accurate {\em globally} optimal value with an overall sample complexity of $T(KB + T_cM) = \tilde{\mathcal{O}} \parat{\frac{1}{\epsilon^4}}$, which is the same as PPG-GAIL and FWPG-GAIL. Comparison of \Cref{thm:trpo1} and \Cref{thm:npg} indicates that TRPO-GAIL has a better sample complexity than NPG-GAIL, mainly because TRPO can update the policy parameter based on an analytical form, which saves the samples that NPG uses for estimating the natural gradient by solving the quadratic optimization problem.

%% file: appendix.tex
\section{Q-sampling and NPG-GAIL Algorithms}\label{app:alg}
In this section, we provide the formal description for the algorithm EstQ in \Cref{alg:EstQ}, which returns an unbiased estimation of the state-action value function (Q-value), and the algorithm of policy update of NPG-GAIL in \Cref{alg:npggail}.
\begin{algorithm}
	\caption{EstQ \cite{zhang2019global}} \label{alg:EstQ} 
	\begin{algorithmic}[1]
		\STATE 	{\bf Input:}   $s,a,\theta$. Initialize $\hat{Q}=0,s^q_1=s,a^q_1=a$
		\STATE Draw $T\sim\text{Geom}(1-\gamma^{1/2})$
		\FOR{ $t=1, 2, \ldots, T-1 $}
		\STATE Collect reward $R(s^q_t,a^q_t)$ and update the Q-function $\hat{Q}\leftarrow\hat{Q}+\gamma^{t/2} R(s^q_t,a^q_t)$
		\STATE Sample $s^q_{t+1}\sim\mathbb{P}(\cdot|s^q_t,a^q_t),a^q_{t+1}\sim\pi_\theta(\cdot|s^q_{t+1})$
		\ENDFOR
		\STATE Collect reward $R(s^q_T,a^q_T)$ and update the Q-function $\hat{Q}\leftarrow\hat{Q}+\gamma^{T/2} R(s^q_T,a^q_T)$
		\STATE 	{ {\bf Output:} $\hat Q^{\pi_\theta}\leftarrow\hat{Q} $}
	\end{algorithmic}
\end{algorithm}

\begin{algorithm}[!htb]
	\caption{Policy update in NPG-GAIL}
	\label{alg:npggail}
	\begin{algorithmic}
		\STATE {\bfseries Input:} Policy parameter $\theta_t$, reward parameter $\alpha_t$, stepsize $\beta_W$, policy stepsize $\eta$, batch-size $M$ and trajectory length $T_c$ 
		\FOR{$i=0,\cdots,MT_c$}
		\STATE $s_i\sim \widetilde{\mathsf{P}}(\cdot |s_{i-1},a_{i-1})$
		\STATE Sample $a_i$ and $a^\prime_i$ independently from $\pi_{\theta_t}(\cdot|s_i)$
		\ENDFOR
		\STATE Initialize $W_0 = 0$
		\FOR{$k=0,\cdots,T_c-1$}
		\FOR{$i=kM,\cdots,(k+1)M-1$}
		\STATE Obtain Q-function estimation $\widehat{Q}(s_i,a_i)$ with reward function $r_{\alpha_t}$  by \Cref{alg:EstQ}.
		\STATE $\hat{g}_{i}=(-\nabla_{\theta_t} \log(\pi_{\theta_t}(a_i|s_i))^\top W_k + \widehat{Q}(s_i,a_i))\nabla_{\theta_t} \log(\pi_{\theta_t}(a_i|s_i)$ \\ \qquad\qquad\qquad\quad $-\widehat{Q}(s_i,a_i) \nabla_{\theta_t}\log(\pi_{\theta_t}(a^\prime_i|s_i)) - \lambda W_k$
		\ENDFOR
		\STATE $\hat{G}_k = \frac{1}{M}\sum_{i=kM}^{(k+1)M-1}\hat{g}_{i}$
		\STATE $W_{k+1}=W_k+\beta_W \hat{G}_k$
		\ENDFOR 
		\STATE $w_t = W_{T_c}$
		\STATE {\bfseries Return:} $\theta_{t+1}=\theta_t - \eta w_t$  
	\end{algorithmic}
\end{algorithm}

\section{Proof of \Cref{lemma:lipschitzcondtion}}
In this section, we first provide two useful lemmas, which establish the smoothness property of the visitation distribution and Q-function.

\begin{lemma}(\cite[Lemma 3]{xu2020improving})\label{lemma:lipchitznu}
	Consider the initial distribution $\xi(\cdot)$ and the transition kernel $\mathsf{P}(\cdot|s,a)$. Let $\xi(\cdot)$ be $\zeta(\cdot)$ or $\mathsf{P}(\cdot|\hat{s}, \hat{a})$ for any given $\hat{s}\in\Sc, \hat{a}\in\Ac$. Denote $\nu_{\pi_{\theta},\xi}$ as the state-action visitation distribution of MDP with policy $\pi_\theta$ and the initialization distribution $\xi$. Suppose \Cref{aspt:ergodic} holds. Then we have, under direct parameterization for any $\theta_1, \theta_2 \in\Theta_p$, 
	\begin{equation*}
	\norm{\nu_{\pi_{\theta},\xi} - \nu_{\pi_{\theta^\prime},\xi}}_{TV} \le C_\nu \norm{\theta_1 -\theta_2}_2,
	\end{equation*}
	where $C_\nu = \frac{\sqrt{|\Ac|}}{2}\parat{1 + \ceiling{\log_{\rho} C_M^{-1}} + (1-\rho)^{-1}}$. 
\end{lemma}
\begin{lemma}(\cite[Lemma 4]{xu2020improving})\label{lemma:lipschitzq}
	Suppose \Cref{aspt:ergodic,assp:rewardassumption} hold. Let $Q^\pi_{\alpha}$ denote the Q-function of policy $\pi$ under the reward function $r_{\alpha}$. For any state-action pair $(s,a) \in \Sc\times\Ac$, $\alpha\in\Lambda$ and $\theta_1, \theta_2\in \Theta_p$ (under direct parameterization), we have 
	\begin{align*}
	|Q^{\pi_{\theta_1}}_{\alpha}(s,a) - Q^{\pi_{\theta_2}}_{\alpha}(s,a)| \le L_Q \lt{\theta_1 -\theta_2},
	\end{align*}
	where $L_Q = \frac{2C_rC_\alpha C_\nu}{1-\gamma}$ and $C_\nu$ is defined in \Cref{lemma:lipchitznu}. 
\end{lemma}

	Denote $d_{\pi}(s) = (1-\gamma)\sum_{t=0}^\infty \gamma^t \prob\left\{s_t = s|\pi \right\}$ as the state visitation distribution induced by policy $\pi$. We next prove \Cref{lemma:lipschitzcondtion} to characterize the Lipschitz constants $L_{11}$, $L_{12}$, $L_{21}$ and $L_{22}$, respectively. 
	\begin{proof}[Proof of \Cref{lemma:lipschitzcondtion}]
		We consider the first inequality in \Cref{lemma:lipschitzcondtion}:
		\begin{align}
		&\lt{\gdt F(\theta_1, \alpha_1) - \gdt F(\theta_2, \alpha_2)}\nonumber\\
		&\quad= \lt{\gdt F(\theta_1, \alpha_1) - \gdt F(\theta_2, \alpha_1) + \gdt F(\theta_2, \alpha_1) - \gdt F(\theta_2, \alpha_2)}\nonumber\\
		&\quad\le \underbrace{\lt{\gdt F(\theta_1, \alpha_1) - \gdt F(\theta_2, \alpha_1)}}_{T_1} + \underbrace{\lt{\gdt F(\theta_2, \alpha_1) - \gdt F(\theta_2, \alpha_2)}}_{T_2}.\label{eq:firstinequality}
		\end{align}
		Next, we upper-bound the terms $T_1$ and $T_2$ in \cref{eq:firstinequality}, respectively.
		
		{\bf Upper-bounding $T_1$}: For any given state-action pair $(s,a)\in \Sc\times\Ac$, we have 
		\begin{align}
		&\left|\left(\gdt F(\theta_1, \alpha_1) - \gdt F(\theta_2, \alpha_1)\right)_{s,a}\right| \nonumber\\
		&\quad\overset{(i)}= \left|\frac{1}{1-\gamma}\left(d_{\pi_{\theta_1}}(s)Q^{\pi_{\theta_1}}_{\alpha_1}(s,a) - d_{\pi_{\theta_2}}(s)Q^{\pi_{\theta_2}}_{\alpha_1}(s,a)\right)\right|\nonumber\\
		& \quad\le\left|\frac{1}{1-\gamma}\parat{(d_{\pi_{\theta_1}}(s) - d_{\pi_{\theta_2}}(s))Q^{\pi_{\theta_1}}_{\alpha_1}(s,a)}\right| + \abs{\frac{1}{1-\gamma}\parat{ d_{\pi_{\theta_2}}(s)(Q^{\pi_{\theta_1}}_{\alpha_1}(s,a) -Q^{\pi_{\theta_2}}_{\alpha_1}(s,a))}}\nonumber\\
		&\quad\overset{(ii)}\le \frac{R_{max}}{(1-\gamma)^2} |d_{\pi_{\theta_1}}(s) - d_{\pi_{\theta_2}}(s)| + \frac{L_Q}{1-\gamma} d_{\pi_{\theta_2}}(s) \lt{\theta_1- \theta_2}, \label{eq:entrywisebound1}
		\end{align}
		where $(i)$ follows from the fact that $\frac{\partial F(\theta, \alpha_1)}{\partial \theta_{s,a}} = -\frac{\partial V(\pi_\theta, \alpha_1)}{\partial \theta_{s,a}} = -\frac{1}{1-\gamma}d_{\pi_{\theta}}(s)Q^{\pi_{\theta}}_{\alpha_1}(s,a)$, and $(ii)$ follows from \Cref{lemma:lipschitzq}.
		Then, we proceed as follows:
		\begin{align*}
		&\lt{\gdt F(\theta_1, \alpha_1) - \gdt F(\theta_2, \alpha_1)}\\
		&\quad = \sqrt{\sum_{s,a} \left|\left(\gdt F(\theta_1, \alpha_1) - \gdt F(\theta_2, \alpha_1)\right)_{s,a}\right|^2}\nonumber\\
		&\quad\overset{(i)}\le \sqrt{\sum_{s,a} \parat{\frac{R_{max}}{(1-\gamma)^2} \abs{d_{\pi_{\theta_1}}(s) - d_{\pi_{\theta_2}}(s)} + \frac{L_Q}{1-\gamma} d_{\pi_{\theta_2}}(s) \lt{\theta_1- \theta_2}}^2} \\
		&\quad\le\sqrt{2|\Ac|}\sqrt{\sum_{s} \parat{\frac{R_{max}}{(1-\gamma)^2} \abs{d_{\pi_{\theta_1}}(s) - d_{\pi_{\theta_2}}(s)}}^2} + \sqrt{2|\Ac|}\sqrt{\sum_{s} \parat{\frac{L_Q}{1-\gamma} d_{\pi_{\theta_2}}(s) \lt{\theta_1- \theta_2}}^2}\\ 
		&\quad\overset{(ii)}\le \sqrt{2|\Ac|} \parat{\sum_{s} \frac{R_{max}}{(1-\gamma)^2} \abs{d_{\pi_{\theta_1}}(s) - d_{\pi_{\theta_2}}(s)} + \sum_{s} \frac{L_Q}{1-\gamma} d_{\pi_{\theta_2}}(s) \lt{\theta_1- \theta_2}} \\
		&\quad \overset{(iii)}\le \frac{2\sqrt{2}|\Ac|C_rC_\alpha}{(1-\gamma)^2}\parat{1 + \ceiling{\log_{\rho} C_M^{-1}} + (1-\rho)^{-1}} \lt{\theta_1- \theta_2} ,
		\end{align*}
		where $(i)$ follows from \cref{eq:entrywisebound1}, $(ii)$ follows from the fact that $\norm{x}_{2} \le \norm{x}_{1}$, and $(iii)$ follows from  \Cref{lemma:lipchitznu} and from the facts $R_{max}\le C_rC_\alpha$ and $$\sum_{s\in\Sc} \abs{d_{\pi_{\theta_1}}(s) - d_{\pi_{\theta_2}}(s)} = 2\norm{d_{\pi_{\theta_1}} - d_{\pi_{\theta_2}}}_{TV} \le 2\norm{\nu_{\pi_{\theta_1}} - \nu_{\pi_{\theta_2}}}_{TV}.$$
		
		{\bf Upper-bounding $T_2$}: For any given state-action pair $(s,a)\in \Sc\times\Ac$, we have 
		\begin{align*}
		\left|\left(\gdt F(\theta_2, \alpha_1) - \gdt F(\theta_2, \alpha_2)\right)_{s,a}\right| &= \left|\frac{1}{1-\gamma}\left(d_{\pi_{\theta_2}}(s)Q^{\pi_{\theta_2}}_{\alpha_1}(s,a) - d_{\pi_{\theta_2}}(s)Q^{\pi_{\theta_2}}_{\alpha_2}(s,a)\right)\right|\\
		&\overset{(i)}= \frac{1}{1-\gamma}d_{\pi_{\theta_2}}(s)\abs{\frac{1}{1-\gamma} \sum_{\hat{s},\hat{a}} \nu_{\pi_{\theta_2},s,a}(\hat{s},\hat{a})(r_{\alpha_1}(\hat{s},\hat{a}) - r_{\alpha_2}(\hat{s},\hat{a}))} \\
		&\overset{(ii)}\le \frac{1}{(1-\gamma)^2} d_{\pi_{\theta_2}}(s) C_r \lt{\alpha_1 -\alpha_2},
		\end{align*}
		where in $(i)$ we denote $\nu_{\pi_{\theta_2},s,a}(\hat{s},\hat{a})$ as the visitation distribution of the Markov chain with initial distribution $\mathsf{P}(\cdot| s_0 =s, a_0 =a)$ and policy $\pi_{\theta_2}$, and $(ii)$ follows from the fact that $|r_{\alpha_1}(\hat{s},\hat{a}) - r_{\alpha_2}(\hat{s},\hat{a})| = |\inner{\gda r_{\alpha^\prime}(\hat{s},\hat{a})}{\alpha_1 -\alpha_2}|\le\lt{\gda r_{\alpha^\prime}(\hat{s},\hat{a})}\lt{\alpha_1 -\alpha_2}\le C_r\lt{\alpha_1 -\alpha_2}$, for some $\alpha^\prime \in [\alpha_1, \alpha_2]$. The inequality above implies that
		\begin{align*}
		\lt{\gdt F(\theta_2, \alpha_1) - \gdt F(\theta_2, \alpha_2)} &=  \sqrt{\sum_{s,a} \left|\left(\gdt F(\theta_2, \alpha_1) - \gdt F(\theta_2, \alpha_2)\right)_{s,a}\right|^2}\\
		&\le \sqrt{\sum_{s,a} \parat{\frac{1}{(1-\gamma)^2} d_{\pi_{\theta_2}}(s) C_r \lt{\alpha_1 -\alpha_2}}^2}\\
		&= \frac{\sqrt{|\Ac|}C_r}{(1-\gamma)^2}   \lt{\alpha_1 -\alpha_2}\sqrt{ \sum_s \parat{d_{\pi_{\theta_2}}(s)}^2}\\
		&\overset{(i)}\le \frac{\sqrt{|\Ac|}C_r}{(1-\gamma)^2}  \lt{\alpha_1 -\alpha_2},
		\end{align*}
		where $(i)$ follows from the fact that $\sqrt{\sum_s \parat{d_{\pi_{\theta_2}}(s)}^2} \le \norm{d_{\pi_{\theta_2}}}_1 =1$.
		
		Therefore we obtain the upper bound of \cref{eq:firstinequality} as follows:
		\begin{align*}
		&\lt{\gdt F(\theta_1, \alpha_1) - \gdt F(\theta_2, \alpha_2)}\\
		&\quad\le \frac{2\sqrt{2}|\Ac|C_rC_\alpha}{(1-\gamma)^2}\parat{1 + \ceiling{\log_{\rho} C_M^{-1}} + (1-\rho)^{-1}} \lt{\theta_1- \theta_2}  +  \frac{\sqrt{|\Ac|}C_r}{(1-\gamma)^2}  \lt{\alpha_1 -\alpha_2},
		\end{align*}
		which determines the constants $L_{11}$ and $L_{12}$.
		
		We then proceed to prove the second inequality in \Cref{lemma:lipschitzcondtion}. 
		\begin{align}
			&\lt{\gda F(\theta_1, \alpha_1) - \gda F(\theta_2, \alpha_2)} \nonumber\\
			&\quad \le \lt{\gda F(\theta_1, \alpha_1) - \gda F(\theta_2, \alpha_1) + \gda F(\theta_2, \alpha_1)  - \gda F(\theta_2, \alpha_2)} \nonumber\\
			&\quad \le \underbrace{\lt{\gda F(\theta_1, \alpha_1) - \gda F(\theta_2, \alpha_1)}}_{T_3} + \underbrace{\lt{\gda F(\theta_2, \alpha_1)  - \gda F(\theta_2, \alpha_2)}}_{T_4}. \label{eq:secondinequality}
		\end{align}
		Next, we upper-bound $T_3$ and $T_4$ in \cref{eq:secondinequality}, respectively.
		
		{\bf Upper-bounding $T_3$}: For any given $1\le i\le q$, we have
		\begin{align*}
		&|(\gda F(\theta_1, \alpha_1) - \gda F(\theta_2, \alpha_1))_{i}| \\
		&= |(\gda V(\pi_E, r_{\alpha_1}) - \gda V(\pi_{\theta_1}, r_{\alpha_1}) - \gda \psi(\alpha_1)  - (\gda V(\pi_E, r_{\alpha_1}) - \gda V(\pi_{\theta_2}, r_{\alpha_1}) - \gda \psi(\alpha_1)))_{i}| \\
		&= |(\gda V(\pi_{\theta_2}, r_{\alpha_1})-\gda V(\pi_{\theta_1}, r_{\alpha_1}))_{i}|\\
		&= \frac{1}{1-\gamma} \abs{\sum_{s,a} (\nu_{\pi_{\theta_1}}(s,a) - \nu_{\pi_{\theta_2}}(s,a))(\gda r_{\alpha_1})_i} \le \frac{\norm{\nu_{\pi_{\theta_1}} - \nu_{\pi_{\theta_2}}}_1 \norm{\frac{\partial r_\alpha}{\partial \alpha_{i}}}_\infty}{1-\gamma}\\
		&\overset{(i)}\le \frac{2C_{\nu}\lt{\theta_1- \theta_2} \norm{\frac{\partial r_\alpha}{\partial \alpha_{i}}}_\infty}{1-\gamma},
		\end{align*}
		where $(i)$ follows from \Cref{lemma:lipchitznu} and the fact that $\norm{p-q}_1 = 2\norm{p-q}_{TV}$. 
		The inequality above further implies that
		\begin{align*}
		\lt{\gda F(\theta_1, \alpha) - \gda F(\theta_2, \alpha)}&\le \frac{2C_{\nu}\lt{\theta_1- \theta_2}}{1-\gamma}\sqrt{\sum_{i=1}^q \norm{\frac{\partial r_\alpha}{\partial \alpha_{i}}}_\infty^2}\\
		&\le\frac{C_r\sqrt{|\Ac|}}{1-\gamma}\parat{1 + \ceiling{\log_{\rho} C_M^{-1}} + (1-\rho)^{-1}} \lt{\theta_1 - \theta_2}.
		\end{align*} 
		
	    {\bf Upper-bounding $T_4$}: We provide a proof for the general parameterization of policy, which includes the direct parameterization of policy as a special case and covers the  last claim of \Cref{lemma:lipschitzcondtion}. We proceed as follows:
		\begin{align*}
		&\lt{\gda F(\theta_2, \alpha_1) - \gda F(\theta_2, \alpha_2)} \nonumber\\
		&\le \lt{\gda V(\pi_E, r_{\alpha_1}) - \gda V(\pi_{\theta_2}, r_{\alpha_1}) - \gda \psi(\alpha_1)  - (\gda V(\pi_E, r_{\alpha_2}) - \gda V(\pi_{\theta_2}, r_{\alpha_2}) - \gda \psi(\alpha_2))}\nonumber \\
		&\le \resizebox{0.7\hsize}{!}{$\frac{1}{1-\gamma}\left( \lt{\int(\gda r_{\alpha_1} - \gda r_{\alpha_2})d\nu_{\pi_E}} +\lt{\int(\gda r_{\alpha_1} - \gda r_{\alpha_2})d\nu_{\pi_{\theta}}}\right)$}\nonumber\\
		&\qquad+ \lt{\gda \psi(\alpha_1) -\gda \psi(\alpha_2)}\\
		&=\resizebox{\hsize}{!}{$\frac{1}{1-\gamma}\left(\sqrt{\sum_{i=1}^q \left(\int(\gda r_{\alpha_1}(s,a) - \gda r_{\alpha_2}(s,a))_id\nu_{\pi_E}\right)^2} + \sqrt{\sum_{i=1}^q \left(\int(\gda r_{\alpha_1}(s,a) - \gda r_{\alpha_2}(s,a))_id\nu_{\pi_{\theta_2}}\right)^2}\right)$}\\
		&\quad+ \lt{\gda \psi(\alpha_1) -\gda \psi(\alpha_2)}\\
		&\overset{(i)}\le\left(\frac{2\sqrt{q}L_r}{1-\gamma} + L_\psi\right) \lt{\alpha_1- \alpha_2},
		\end{align*}
		where $(i)$ follows from \Cref{assp:regularizer} and further  because for any $(s,a)$ and $i$, we have $$|(\gda r_{\alpha_1}(s,a) - \gda r_{\alpha_2}(s,a))_i| \le \lt{\gda r_{\alpha_1}(s,a) - \gda r_{\alpha_2}(s,a)} \le L_r \lt{\alpha_1 - \alpha_2}.$$
		
		Therefore, we obtain the following upper bound in \cref{eq:secondinequality}
		\begin{align*}
		&\lt{\gda F(\theta_1, \alpha_1) - \gda F(\theta_2, \alpha_2)}\\
		&\quad\le \frac{C_r\sqrt{|\Ac|}}{1-\gamma}\parat{1 + \ceiling{\log_{\rho} C_M^{-1}} + (1-\rho)^{-1}} \lt{\theta_1 - \theta_2} + \left(\frac{2\sqrt{q}L_r}{1-\gamma} + L_\psi\right) \lt{\alpha_1- \alpha_2},
		\end{align*}
		which determines $L_{21}$ and $L_{22}$.
	\end{proof}

\section{Proof of \Cref{prop:PL-like}}
We define $\tp(\alpha) \defeq \argmin_{\theta \in\Theta_p} F(\theta, \alpha)$. If there exist multiple optimal points, then $\tp(\alpha)$ can be any optimal point.

We first provide a lemma, which characterizes the gradient dominance property for the function $F(\theta, \alpha)$ with a fixed reward parameter $\alpha$. 
\begin{lemma}\label{lemma:rlgradientdomination}(\cite[Lemma 4.1]{agarwal2019optimality})
	For any given $\alpha\in \Lambda$, $F(\theta, \alpha)$ defined in \cref{eq:minmax} with direct parameterization satisfies,
	\begin{align*}
	F(\theta, \alpha) - F(\theta_{op}(\alpha), \alpha) \le C_d \max_{\tilde{\theta}\in\Theta_p} \inner{\theta-\tilde{\theta}}{\gdt F(\theta, \alpha)},
	\end{align*}
	where $C_d$ = $\frac{1}{(1-\gamma)\min_s\left\{\zeta(s)\right\}}$.
\end{lemma}
We then provide the proof of \Cref{prop:PL-like}.
\begin{proof}[Proof of \Cref{prop:PL-like}]
	We proceed as follows:
	\begin{align*}
	g(\theta) - g(\theta^*) & = F(\theta, \ap(\theta)) - F(\theta^*, \ap({\theta^*}))\\
	&=F(\theta, \ap(\theta)) - F(\tp(\ap(\theta)), \ap(\theta)) +  F(\tp(\ap(\theta)), \ap(\theta)) -F(\theta^*, \ap(\theta^*))\\
	&\overset{(i)}\le F(\theta, \ap(\theta)) - F(\tp(\ap(\theta)), \ap(\theta)) \\
	&\overset{(ii)}\le C_d \max_{\bar{\theta}\in \Theta_p} \inner{\theta -\bar{\theta}}{\gd_\theta F(\theta, \ap(\theta))} \\
	&\overset{(iii)}= C_d \max_{\bar{\theta} \in \Theta_p} \inner{\theta -\bar{\theta}}{\gd g(\theta)},
	\end{align*}
	where  $(i)$ follows from the fact that 
	\begin{align*}
	&F(\tp(\ap(\theta)), \ap(\theta)) - F(\theta^*,\ap(\theta^*))\\ 
	&=  \underbrace{F(\tp(\ap(\theta)), \ap(\theta)) - F(\theta^*, \ap(\theta))}_{\leq 0} +  \underbrace{F(\theta^*, \ap(\theta))   -F(\theta^*, \ap(\theta^*))}_{\leq 0}\leq 0,
	\end{align*}
	$(ii)$ follows from \Cref{lemma:rlgradientdomination}, and $(iii)$ follows because $\gd g(\theta) = \gd_\theta F(\theta, \alpha)|_{\alpha=\ap(\theta)}$.
\end{proof}

\section{Supporting Lemmas for GAIL Framework}

In this section, we establish two supporting lemmas that are useful for the proof of our main theorems. 


\begin{lemma}\label{lemma:noniidsample}
	Suppose \Cref{aspt:ergodic} holds. Consider the gradient approximation in the nested-loop GAIL framework (\Cref{alg:nestedloopIL}). For any $k$ and $t$, $0\le k \le K-1$ and $0 \le t \le T-1$, we have
	\begin{align*}
	\eptt\left[\norm{\sgda F(\theta_{t}, \alpha_{k}^t)  - \gda F(\theta_{t}, \alpha_{k}^t)}_2^2\right] \le \frac{16C_r^2}{1-\gamma}\left(1+\frac{C_M}{1-\rho}\right)\frac{1}{B}.
	\end{align*}
\end{lemma}

\begin{proof}[Proof of \Cref{lemma:noniidsample}] 

We denote $d_{\pi}(s) \defeq (1-\gamma)\sum_{t=0}^\infty \gamma^t \prob\left\{s_t = s\right\}$ as the state visitation distribution of the Markov chain with initial distribution $\zeta(\cdot)$, transition kernel $\mathsf{P}(\cdot|s,a)$ and policy $\pi$. Both trajectories $(s_0^E, a_0^E, s_1^E, a_1^E, \cdots, s_i^E, a_i^E)$ and $(s_0^\theta, a_0^\theta, s_1^\theta ,a_1^\theta, \cdots, s_i^E, a_i^E)$ are sampled under the transition kernel $\tilde{\mathsf{P}}(\cdot|s,a) = \gamma\mathsf{P}(\cdot|s,a) + (1-\gamma)\zeta(\cdot)$. Recall that it has been shown in \cite{konda2002actor} that the stationary distribution of the Markov chain with transition kernel and policy $\pi$ is $d_\pi$.

By definition,  we have,
\begin{align}
	&\eptt\left[\norm{\sgda F(\theta_{t}, \alpha_{k}^t)  - \gda F(\theta_{t}, \alpha_{k}^t)}_2^2\right]\nonumber\\
	&\quad =\resizebox{.95\hsize}{!}{$\eptt\left[\norm{\frac{1}{(1-\gamma)B} \left(\sum_{i=0}^{B-1} \nabla_{\alpha_{k}^t} r_{\alpha_k^t} (s_i^E, a_i^E) - \nabla_{\alpha_{k}^t} r_{\alpha_k^t} (s_i^{\theta}, a_i^{\theta})\right)   - \frac{1}{1-\gamma}\left(\eptt_{(s,a) \sim\nu_{\pi_{E}}}\left[\nabla_{\alpha_{k}^t} r_{\alpha_k^t}(s,a)\right] - \eptt_{(s,a) \sim\nu_{\pi_{\theta_{t}}}}\left[\nabla_{\alpha_{k}^t} r_{\alpha_k^t}(s,a)\right]\right)}_2^2 \right]$}\nonumber\\
	&\quad \le\resizebox{.65\hsize}{!}{$ \frac{2}{(1-\gamma)^2B^2}\underbrace{\eptt\left[\norm{\sum_{i=0}^{B-1}\left(\nabla_{\alpha_{k}^t} r_{\alpha_k^t} (s_i^E, a_i^E)  -  \eptt_{(s,a) \sim\nu_{\pi_{E}}}\left[\nabla_{\alpha_{k}^t} r_{\alpha_k^t}(s,a)\right]\right)}_2^2\right]}_{T_1}$}\nonumber\\
	&\qquad + \resizebox{.65\hsize}{!}{$ \frac{2}{(1-\gamma)^2B^2}\underbrace{\eptt\left[\norm{\sum_{i=0}^{B-1}\left(\nabla_{\alpha_{k}^t} r_{\alpha_k^t} (s_i^{\theta}, a_i^{\theta}) - \eptt_{(s,a) \sim\nu_{\pi_{\theta_{t}}}}\left[\nabla_{\alpha_{k}^t} r_{\alpha_k^t}(s,a)\right] \right)}_2^2\right]}_{T_2} $}. \label{eq:noniid}	
\end{align}
We first provide an upper bound on the term $T_1$ in \cref{eq:noniid}, and proceed as follows:
\begin{align}
T_1&=\resizebox{.50\hsize}{!}{$\sum_{i=0}^{B-1} \eptt\norm{ \nabla_{\alpha_{k}^t} r_{\alpha_k^t} (s_i^E, a_i^E)  -  \eptt_{(s,a) \sim\nu_{\pi_{E}}}\left[\nabla_{\alpha_{k}^t} r_{\alpha_k^t}(s,a)\right]}_2^2$} \nonumber\\
&\quad + \resizebox{.9\hsize}{!}{$\sum_{i\neq j}\eptt \inner{\nabla_{\alpha_{k}^t} r_{\alpha_k^t} (s_i^E, a_i^E)  -  \eptt_{(s,a) \sim\nu_{\pi_{E}}}\left[\nabla_{\alpha_{k}^t} r_{\alpha_k^t}(s,a)\right]}{\nabla_{\alpha_{k}^t} r_{\alpha_k^t} (s_j^E, a_j^E)  -  \eptt_{(s,a) \sim\nu_{\pi_{E}}}\left[\nabla_{\alpha_{k}^t} r_{\alpha_k^t}(s,a)\right]}$}\nonumber\\
& \le \resizebox{.93\hsize}{!}{$4BC_r^2 + \sum_{i\neq j} \eptt \inner{\nabla_{\alpha_{k}^t} r_{\alpha_k^t} (s_i^E, a_i^E)  -  \eptt_{(s,a) \sim\nu_{\pi_{E}}}\left[\nabla_{\alpha_{k}^t} r_{\alpha_k^t}(s,a)\right]}{\nabla_{\alpha_{k}^t} r_{\alpha_k^t} (s_j^E, a_j^E)  -  \eptt_{(s,a) \sim\nu_{\pi_{E}}}\left[\nabla_{\alpha_{k}^t} r_{\alpha_k^t}(s,a)\right]}$}\label{eq:middleboundinga}
\end{align}
Define the filtration $\mathcal{F}_i = \sigma(s_0^E, a_0^E, s_1^E, a_1^E, \cdots, s_i^E, a_i^E)$. We continue to bound the second term in \cref{eq:middleboundinga} as follows:
\begin{align}
	&\resizebox{.85\hsize}{!}{$\eptt\left[\inner{\nabla_{\alpha_{k}^t} r_{\alpha_k^t} (s_i^E, a_i^E)  -  \eptt_{(s,a) \sim\nu_{\pi_{E}}}\left[\nabla_{\alpha_{k}^t} r_{\alpha_k^t}(s,a)\right]}{\nabla_{\alpha_{k}^t} r_{\alpha_k^t} (s_j^E, a_j^E)  -  \eptt_{(s,a) \sim\nu_{\pi_{E}}}\left[\nabla_{\alpha_{k}^t} r_{\alpha_k^t}(s,a)\right]}\right]$}  \nonumber\\
	&\quad= \resizebox{.9\hsize}{!}{$\eptt\left[\eptt\left[\inner{\nabla_{\alpha_{k}^t} r_{\alpha_k^t} (s_i^E, a_i^E)  -  \eptt_{(s,a) \sim\nu_{\pi_{E}}}\left[\nabla_{\alpha_{k}^t} r_{\alpha_k^t}(s,a)\right]}{\nabla_{\alpha_{k}^t} r_{\alpha_k^t} (s_j^E, a_j^E)  -  \eptt_{(s,a) \sim\nu_{\pi_{E}}}\left[\nabla_{\alpha_{k}^t} r_{\alpha_k^t}(s,a)\right]}\middle| \mathcal{F}_i\right]\right]$}\nonumber\\
	&\quad= \resizebox{.9\hsize}{!}{$\eptt\left[\inner{\nabla_{\alpha_{k}^t} r_{\alpha_k^t} (s_i^E, a_i^E)  -  \eptt_{(s,a) \sim\nu_{\pi_{E}}}\left[\nabla_{\alpha_{k}^t} r_{\alpha_k^t}(s,a)\right]}{\eptt\left[\nabla_{\alpha_{k}^t} r_{\alpha_k^t} (s_j^E, a_j^E)\middle| \mathcal{F}_i\right]  -  \eptt_{(s,a) \sim\nu_{\pi_{E}}}\left[\nabla_{\alpha_{k}^t} r_{\alpha_k^t}(s,a)\right]}\right]$}\nonumber\\
	&\quad\le \resizebox{.9\hsize}{!}{$\eptt\left[\norm{\nabla_{\alpha_{k}^t} r_{\alpha_k^t} (s_i^E, a_i^E)  -  \eptt_{(s,a) \sim\nu_{\pi_{E}}}\left[\nabla_{\alpha_{k}^t} r_{\alpha_k^t}(s,a)\right]}_2\norm{\eptt\left[\nabla_{\alpha_{k}^t} r_{\alpha_k^t} (s_j^E, a_j^E)\middle|\mathcal{F}_i\right]  -  \eptt_{(s,a) \sim\nu_{\pi_{E}}}\left[\nabla_{\alpha_{k}^t} r_{\alpha_k^t}(s,a)\right]}_2\right]$}\nonumber\\
	&\quad\le \resizebox{.6\hsize}{!}{$2C_r \eptt\left[\norm{\eptt\left[\nabla_{\alpha_{k}^t} r_{\alpha_k^t} (s_j^E, a_j^E)\middle| \mathcal{F}_i\right]  -  \eptt_{(s,a) \sim\nu_{\pi_{E}}}\left[\nabla_{\alpha_{k}^t} r_{\alpha_k^t}(s,a)\right]}_2\right]$}\nonumber\\
	&\quad = \resizebox{.8\hsize}{!}{$2C_r\eptt\norm{\int_{s\sim \mathbb{P}(s_j\in\cdot|s_i^E, a_i^E), a\sim\pi_{E}(\cdot|s)} \nabla_{\alpha_{k}^t} r_{\alpha_{k}^t}(s,a)dsda -\int_{s\sim\chi_\theta, a\sim\pi_{E}(\cdot|s)}\nabla_{\alpha_{k}^t} r_{\alpha_{k}^t}(s,a)dsda}_2$} \nonumber\\
    &\quad = \resizebox{.85\hsize}{!}{$2C_r\eptt\sqrt{\sum_{l=1}^{q}\left(\int_{s\sim \mathbb{P}(s_j\in\cdot|s_i^E, a_i^E), a\sim\pi_{E}(\cdot|s)} \frac{\partial r_{\alpha}}{\partial \alpha_l}|_{\alpha=\alpha_k^t}(s,a) dsda -\int_{s\sim\chi_\theta, a\sim\pi_{E}(\cdot|s)}\frac{\partial r_{\alpha}}{\partial \alpha_l}|_{\alpha=\alpha_k^t}(s,a)dsda\right)^2}$}\nonumber\\
    &\quad \overset{(i)}\le 2C_r\eptt\sqrt{\sum_{l=1}^{q}\left(\norm{\frac{\partial r_\alpha}{\partial \alpha_i}}_\infty \dtv \left(\mathbb{P}(s_j\in\cdot|s_i =s_i^E, a_i = a_i^E), \chi_{\pi_{E}}\pi_E\right)\right)^2},\label{eq:middleboundinga2}
\end{align}
where $(i)$ follows from the fact that $|\int fd\mu - \int fd\nu|\le \norm{f}_\infty\dtv(\mu, \nu)$. We next derive a bound on the total variation distance in the above equation as follows.
\begin{align}
	&\dtv \left(\mathbb{P}(s_j\in\cdot, a_j \in\cdot|s_i =s_i^E, a_i = a_i^E), \chi_{\pi_E} \pi_E\right)\nonumber \\
	&\quad = \dtv \left(\mathbb{P}(s_j\in\cdot|s_i =s_i^E, a_i = a_i^E), \chi_{\pi_{E}}\right)\nonumber\\
	&\quad = \dtv \left(\int_{s}\mathbb{P}(s_j\in\cdot|s_{i+1} =s)d\tilde{\mathsf{P}}(s|s_i =s_i^E, a_i = a_i^E), \chi_{\pi_{E}}\right)\nonumber\\
	&\quad \le \int_{s} \dtv \left(\mathbb{P}(s_j\in\cdot|s_{i+1} =s), \chi_{\pi_{E}}\right)d\tilde{\mathsf{P}}(s|s_i =s_i^E, a_i = a_i^E)\nonumber\\
	&\quad \overset{(i)}\le \int_{s} C_M\rho^{j-i-1}d\tilde{\mathsf{P}}(s|s_i =s_i^E, a_i = a_i^E) = C_M\rho^{j-i-1}, \label{eq:middleboundinga3}
\end{align} 
where $(i)$ follows from \Cref{aspt:ergodic}.
Substituting \cref{eq:middleboundinga3} into \cref{eq:middleboundinga2} and then further into \cref{eq:middleboundinga} yields the following upper-bound on $T_1$ 
\begin{align}
T_1&\le  4BC_r^2 + 2\sum_{i=0}^{B-2}\sum_{j= i+1}^{B-1} 2C_MC_r^2\rho^{j-i-1}\le 4BC_r^2 (1+\frac{C_M}{1-\rho}). \label{eq:boundinga}
\end{align}
By following steps similar to those from \cref{eq:middleboundinga,eq:middleboundinga2,eq:middleboundinga3,eq:boundinga}, we can show that
$$T_2\le4BC_r^2 (1+\frac{C_M}{1-\rho}).$$ Therefore, we have
\begin{align*}
	\eptt\left[\norm{\sgda F(\theta_{t}, \alpha_{k}^t)  - \gda F(\theta_{t}, \alpha_{k}^t)}_2^2\right] \le \frac{16C_r^2}{(1-\gamma)^2}\left(1+\frac{C_M}{1-\rho}\right)\frac{1}{B}.
\end{align*}
\end{proof}

%
%

\begin{lemma}\label{thm:alphaupdate}
Suppose \Cref{aspt:ergodic,assp:rewardassumption} hold. Consider \Cref{alg:nestedloopIL} with $\alpha$-update stepsize $\beta = \frac{\mu}{4L_{22}^2}$. For any $0\le t\le T-1$, we have 
 $$\eptt\left[\norm{\alpha_{K}^t -\alpha_{op}(\theta_{t})}_2^2\right]  \le C_\alpha^2 e^{-\frac{\mu^2}{8L_{22}^2} K} + \frac{48C_r^2}{\mu^2(1-\gamma)^2}(1+\frac{ C_M}{1-\rho})\frac{1}{B}.$$ 
 Let $K \ge \frac{8L_{22}^2}{\mu^2}\log \frac{2C_\alpha^2}{\Delta_\alpha}$ and $B\ge \frac{96C_r^2}{\mu^2(1-\gamma)^2}\left(1 +\frac{C_M}{1-\rho}\right)\frac{1}{\Delta_\alpha}$, we have $\eptt\left[\norm{\alpha_{K}^t -\alpha_{op}(\theta_{t})}_2^2\right]\le \Delta_\alpha$. The expected total computational complexity is given by $$KB = \bigO{\frac{1}{(1-\gamma)^2\Delta_\alpha}\log\left(\frac{1}{\Delta_\alpha}\right)}.$$
\end{lemma}
\begin{proof}[Proof of \Cref{thm:alphaupdate}]
	We proceed as follows:
	\begin{align}
	&\norm{\alpha_{k+1}^t -\alpha_{op}(\theta_{t})}_2^2 \nonumber\\
	&\quad\overset{(i)}\le \norm{\alpha_k^t + \beta\sgda F(\theta_t, \alpha_{k}^t) - \alpha_{op}(\theta_{t})}_2^2  \nonumber\\
	&\quad=\norm{\alpha_{k}^t -\alpha_{op}(\theta_{t})}_2^2 + \beta^2\norm{\sgda F(\theta_t, \alpha_{k}^t)}_2^2 + 2\beta \inner{\sgda F(\theta_t, \alpha_{k}^t)}{\alpha_{k}^t -\alpha_{op}(\theta_{t})} \nonumber\\
	&\quad\overset{(ii)}\le \norm{\alpha_{k}^t -\alpha_{op}(\theta_{t})}_2^2 + 2\beta^2\norm{\gda F(\theta_t, \alpha_{k}^t)}_2^2 + 2\beta^2\norm{\sgda F(\theta_t, \alpha_{k}^t) - \gda F(\theta_t, \alpha_{k}^t)}_2^2\nonumber \\
	&\qquad + 2\beta \inner{\gda F(\theta_t, \alpha_{k}^t)}{\alpha_{k}^t -\alpha_{op}(\theta_{t})} + 2\beta \inner{\sgda F(\theta_t, \alpha_{k}^t) -\gda F(\theta_t, \alpha_{k}^t)}{\alpha_{k}^t -\alpha_{op}(\theta_{t})}\nonumber\\
	&\quad \overset{(iii)}\le (1- 2\beta\mu + 2\beta^2L_{22}^2) \norm{\alpha_{k}^t -\alpha_{op}(\theta_{t})}_2^2 + 2\beta^2\norm{\sgda F(\theta_t, \alpha_{k}^t) -\gda F(\theta_t, \alpha_{k}^t)}_2^2 \nonumber\\
	&\qquad  + 2\beta \inner{\sgda F(\theta_t, \alpha_{k}^t) -\gda F(\theta_t, \alpha_{k}^t)}{\alpha_{k}^t -\alpha_{op}(\theta_{t})}\nonumber \\
	&\quad \overset{(iv)}\le (1+ 2\beta^2 L_{22}^2 - \mu\beta) \norm{\alpha_{k}^t -\alpha_{op}(\theta_{t})}_2^2 + (2\beta^2 + \beta/\mu) \norm{\sgda F(\theta_t, \alpha_{k}^t) -\gda F(\theta_t, \alpha_{k}^t)}_2^2\nonumber\\
	&\quad \overset{(v)}\le \left(1- \frac{\mu^2}{8L_{22}^2}\right)\norm{\alpha_{k}^t -\alpha_{op}(\theta_{t})}_2^2 + \frac{3}{8L_{22}^2} \norm{\sgda F(\theta_t, \alpha_{k}^t) -\gda F(\theta_t, \alpha_{k}^t)}_2^2,\label{eq:distancecontraction}
	\end{align}
	where $(i)$ follows from the non-expansive property of the projection operator, $(ii)$ follows because $\norm{A+B}_2^2 \le 2 \norm{A}_2^2  + 2\norm{B}_2^2$, $(iii)$ follows from \Cref{lemma:lipschitzcondtion} and the fact $\inner{\gda F(\theta_t, \alpha_{k}^t)}{\alpha_{k}^t -\alpha_{op}(\theta_{t})} \le -\mu \norm{\alpha_{k}^t -\alpha_{op}(\theta_{t})}_2^2$,  $(iv)$ follows because 
	\begin{align*}
	&\langle\sgda F(\theta_t, \alpha_{k}^t) -\gda F(\theta_t, \alpha_{k}^t), \alpha_{k}^t -\alpha_{op}(\theta_{t}) \rangle \\
	&\quad\le \frac{\mu}{2}\norm{\alpha_{k}^t -\alpha_{op}(\theta_{t})}_2^2 + \frac{1}{2\mu}\|\sgda F(\theta_t, \alpha_{k}^t) -\gda F(\theta_t, \alpha_{k}^t)\|_2^2,
	\end{align*} 
	and $(v)$ follows by letting $\beta =\frac{\mu}{4L_{22}^2}$ and because $\mu\le L_{22}$.
	
	Applying \cref{eq:distancecontraction} recursively and using the fact $1-x \le e^{-x}$, we obtain
	\begin{align*}
		\norm{\alpha_{K}^t -\alpha_{op}(\theta_{t})}_2^2 &\le e^{-\frac{\mu^2}{8L_{22}^2} K} \norm{\alpha_{0}^t -\alpha_{op}(\theta_{t})}_2^2 \\
		&\quad+ \frac{3}{8L_{22}^2} \sum_{k=0}^{K-1} \left(1- \frac{\mu^2}{8L_{22}^2}\right)^{K-1-k}\norm{\sgda F(\theta_t, \alpha_{k}^t) -\gda F(\theta_t, \alpha_{k}^t)}_2^2.
	\end{align*}
	
	Then, taking expectation on both sides of above inequality and applying \Cref{lemma:noniidsample} yield
	\begin{align*}
		\eptt\left[\norm{\alpha_{K}^t -\alpha_{op}(\theta_{t})}_2^2\right] &\le C_\alpha^2 e^{-\frac{\mu^2}{8L_{22}^2} K} + \frac{3}{8L_{22}^2}\sum_{k=0}^{K-1}  \left(1- \frac{\mu^2}{8L_{22}^2}\right)^{K-1-k} \frac{16C_r^2}{(1-\gamma)^2}(1+\frac{ C_M}{1-\rho})\frac{1}{B}\\
		&\le C_\alpha^2 e^{-\frac{\mu^2}{8L_{22}^2} K} + \frac{48C_r^2}{\mu^2(1-\gamma)^2}(1+\frac{C_M}{1-\rho})\frac{1}{B},
	\end{align*}
	which completes the proof.
\end{proof}

\section{Proof of \Cref{thm:ppg,thm:fwalg}: Global Convergence of PPG-GAIL and FWPG-GAIL}
In this section, we provide the proof of \Cref{thm:ppg,thm:fwalg}. We first provide three supporting lemmas. Specifically, \Cref{lemma:lipshitzalpha,lemma:gradientlipschitz} establish the smoothness condition of the global optimal $\alpha_{op}(\theta)$ and the gradient $\nabla g(\theta)$. Similar property has also been established in \cite{nouiehed2019solving,lin2020near}. 
\Cref{lemma:pgerrorbound} provides the upper bound on the bias and variance errors introduced by the stochastic gradient estimator of $\gdt F(\theta_t, \alpha_t)$. 
\subsection{Supporting Lemmas}
\begin{lemma}\label{lemma:lipshitzalpha}
	Suppose \Cref{assp:regularizer,assp:rewardassumption,assp:objectivefunction,aspt:ergodic} holds and  the policy  takes the direct parameterization specified in \Cref{sec:gail}. We have $\lt{\ap(\theta_1) - \ap(\theta_2)} \le \frac{L_{21}}{\mu}\lt{\theta_1 - \theta_2}$, where $\ap(\theta)$ is the unique global optimal that satisfies $\ap(\theta) = \argmax_{\alpha\in\Lambda} F(\theta, \alpha)$.
\end{lemma}
\begin{proof}[Proof of \Cref{lemma:lipshitzalpha}]
	Since $F(\theta_1, \alpha)$ is strongly concave on $\alpha$, the following two inequalities hold for all $\alpha \in \Lambda$, 
	\begin{align}
	&F(\theta_1, \ap(\theta_1)) - F(\theta_1, \alpha) \ge \frac{\mu}{2}\lt{\alpha - \ap(\theta_1)}^2,\label{eq:stronglyconcave1}\\ 
	&F(\theta_1, \ap(\theta_1)) - F(\theta_1, \alpha) \le \frac{\lt{\gda F(\theta_1, \alpha)}^2}{2\mu}.\label{eq:stronglyconcave2}
	\end{align}
	In \cref{eq:stronglyconcave1,eq:stronglyconcave2}, letting $\alpha = \ap(\theta_2)$ and using the gradient Lipschitz condition established in \Cref{lemma:lipschitzcondtion}, we have 
	\begin{align*}
	\frac{\mu}{2}\lt{\ap(\theta_2) - \ap(\theta_1)}^2 \le  \frac{\lt{\gda F(\theta_1, \ap(\theta_2))}^2}{2\mu} \le \frac{L_{21}^2\lt{\theta_2- \theta_2}^2}{2\mu},
	\end{align*}
	which implies $\lt{\ap(\theta_1) - \ap(\theta_2)} \le \frac{L_{21}}{\mu}\lt{\theta_1 - \theta_2}$. 
\end{proof}
\begin{lemma}\label{lemma:gradientlipschitz}
	Suppose \Cref{assp:regularizer,assp:rewardassumption,assp:objectivefunction,aspt:ergodic} hold and the policy takes the direct parameterization specified in \Cref{sec:gail}. Then we have $$\gdt g(\theta) = \gdt F(\theta, \alpha)|_{\alpha = \ap(\theta)},$$ and for any $\theta_1,\theta_2\in\Theta_p$,
	$$\lt{\gdt g(\theta_1) - \gdt g(\theta_2)} \le (L_{11}+ (L_{12}L_{21})/\mu)\lt{\theta_1- \theta_2},$$ where $L_{11}$, $L_{12}$ and $L_{21}$ are defined in \Cref{lemma:lipschitzcondtion}. 
\end{lemma}
\begin{proof}[Proof of \Cref{lemma:gradientlipschitz}]	
	Taking the directional derivative of $g(\theta)$ with respect to the direction $\ell$, we have
	\begin{align}
	\frac{\partial g(\theta)}{\partial \ell} &= \lim_{\epsilon \to 0} \frac{g(\theta + \epsilon\ell) -g(\theta)}{\epsilon} = \lim_{\epsilon \to 0} \frac{F(\theta + \epsilon\ell, \ap(\theta + \epsilon\ell)) -F(\theta, \ap(\theta))}{\epsilon} \nonumber\\
	&= \lim_{\epsilon \to 0} \frac{F(\theta + \epsilon\ell, \ap(\theta + \epsilon\ell)) - F(\theta + \epsilon\ell, \ap(\theta ))+ F(\theta + \epsilon\ell, \ap(\theta )) -F(\theta, \ap(\theta))}{\epsilon}\nonumber\\
	& \overset{(i)}= \lim_{\epsilon \to 0}\ell^\top\gda F(\theta, \alpha'_\epsilon)  + \ell^\top\gdt F(\theta, \ap(\theta))\nonumber\\
	&\overset{(ii)}= \ell^\top\gdt F(\theta, \ap(\theta)),\label{eq:middlegradientlipchitz}
	\end{align}	
	where $\alpha'_\epsilon$ in $(i)$ is a point between $\ap(\theta + \epsilon\ell)$ and $\ap(\theta)$, and $(ii)$ follows from \Cref{lemma:lipshitzalpha} and hence we have $\lim_{\epsilon \to 0}\gda F(\theta, \alpha'_\epsilon)  = \gda F(\theta, \ap(\theta)) = 0$.  Since \cref{eq:middlegradientlipchitz} holds for all directions $\ell$, we have $\gdt g(\theta) = \gdt F(\theta, \ap(\theta))$.
	
	We then proceed to prove the gradient Lipschitz condition of $g(\theta_t)$. For any given $\theta_1, \theta_2 \in \Theta_p$, we have 
	\begin{align*}
	&\lt{\gdt g(\theta_1) - \gdt g(\theta_2)}\\
	&\quad= \lt{\gdt F(\theta_1, \ap(\theta_1)) - \gdt F(\theta_2, \ap(\theta_2))}\\
	&\quad =\lt{\gdt F(\theta_1, \ap(\theta_1)) - \gdt F(\theta_1, \ap(\theta_2)) + \gdt F(\theta_1, \ap(\theta_2)) - \gdt F(\theta_2, \ap(\theta_2))} \\
	&\quad \le \lt{\gdt F(\theta_1, \ap(\theta_1)) - \gdt F(\theta_1, \ap(\theta_2))} + \lt{\gdt F(\theta_1, \ap(\theta_2)) - \gdt F(\theta_2, \ap(\theta_2))} \\
	&\quad\le L_{12} \lt{\ap(\theta_1) - \ap(\theta_2)} + L_{11}\lt{\theta_1 - \theta_2} \\
	&\quad\overset{(i)}\le (L_{11} + \frac{L_{12}L_{21}}{\mu})\lt{\theta_1 - \theta_2},
	\end{align*}
	where $(i)$ follows from \Cref{lemma:lipshitzalpha}.
\end{proof}
\begin{lemma}\label{lemma:pgerrorbound}
	Suppose \Cref{aspt:ergodic} holds. For the policy gradient estimation specified in \cref{eq:policygradient}, in each iteration $t$, $0\le t\le T-1$, we have
	\begin{equation*}
	\eptt\left[\norm{\sgdt F(\theta_t, \alpha_t)-\gdt F(\theta_t, \alpha_t)}_2^2\right]\le \frac{4|\Ac|R_{max}^2 }{(1-\gamma^{1/2})^2(1-\gamma)^2}  \left(1+\frac{2C_M\rho}{1-\rho}\right)\frac{1}{b}.
	\end{equation*}
	Let the sample trajectory size $b\ge \frac{4|\Ac|R_{max}^2 }{(1-\gamma^{1/2})^2(1-\gamma)^2}  \left(1+\frac{2C_M\rho}{1-\rho}\right)\frac{1}{\Delta_{\theta}}$, we have $\eptt\left[\norm{\sgdt F(\theta_t, \alpha_t)-\gdt F(\theta_t, \alpha_t)}_2^2\right]\le \Delta_{\theta}.$
\end{lemma}
\begin{proof}[Proof of \Cref{lemma:pgerrorbound}]
	
	We define the vector $g_i\in \mathbb{R}^{|\Sc|\cdot|\Ac|}$ with each entry given by $(g_i)_{s,a}  = -\frac{\hat{Q}(s,a)}{1-\gamma} \indicator{s_i= s}$. Then, we proceed as follows:
	\begin{align}
	&\eptt\left[\norm{\sgdt F(\theta_t, \alpha_t)-\gdt F(\theta_t, \alpha_t)}_2^2\right] = \eptt\left[\norm{\frac{1}{b}\sum_{i=0}^{b-1} (g_i -\gdt F(\theta_t, \alpha_t))}_2^2\right] \nonumber\\
	&\quad=\frac{1}{b^2}\eptt\left[\sum_{i=0}^{b-1} \eptt\norm{g_i-\gdt F(\theta_t, \alpha_t)}_2^2 + \sum_{i\neq j} \eptt \inner{g_i-\gdt F(\theta_t, \alpha_t)}{g_j-\gdt F(\theta_t, \alpha_t)}\right]\nonumber\\
	&\quad\overset{(i)}\le \frac{4|\Ac|R_{max}^2 }{b(1-\gamma^{1/2})^2(1-\gamma)^2} + \frac{2}{b^2} \sum_{i=1}^{b-2} \sum_{j=i+1}^{b-1} \underbrace{\eptt\left[\inner{g_i-\gdt F(\theta_t, \alpha_t)}{g_j-\gdt F(\theta_t, \alpha_t)}\right]}_{T_1},\label{eq:midpgerrorbound}
	\end{align}
	where $(i)$ follows from the facts that $\norm{g_i}_2 = \left|\frac{\sqrt{|\Ac|}\hat{Q}(s_i, a_i)}{1-\gamma}\right| \le\frac{\sqrt{|\Ac|}R_{max}}{(1-\gamma^{1/2})(1-\gamma)}$ and $\norm{\gdt F(\theta_t, \alpha_t)}_2\le \frac{\sqrt{|\Ac|}R_{max}}{(1-\gamma)^2}\le \frac{\sqrt{|\Ac|}R_{max}}{(1-\gamma^{1/2})(1-\gamma)}$.

Define the filtration $\mathcal{F}_i = \sigma\left(s_0, s_1, \cdots, s_i\right)$. For the term $T_1$ in \cref{eq:midpgerrorbound} with $i<j$, we have 
	\begin{align}
	&\eptt\left[\inner{g_i-\gdt F(\theta_t, \alpha_t)}{g_j-\gdt F(\theta_t, \alpha_t)}\right]\\
	&\quad= \eptt\left[\eptt\left[\inner{g_i-\gdt F(\theta_t, \alpha_t)}{g_j-\gdt F(\theta_t, \alpha_t)}\middle| \mathcal{F}_i\right]\right]\nonumber\\
	&\quad= \eptt\left[\inner{g_i-\gdt F(\theta_t, \alpha_t)}{\eptt\left[g_j-\gdt F(\theta_t, \alpha_t)\middle| \mathcal{F}_i\right]}\right]\nonumber\\
	&\quad\le \eptt\left[\norm{g_i-\gdt F(\theta_t, \alpha_t)}_2\norm{\eptt\left[g_j-\gdt F(\theta_t, \alpha_t)\middle| \mathcal{F}_i\right]}_2\right]\nonumber\\
	&\quad\le \frac{2R_{max}\sqrt{|\Ac|}}{(1-\gamma)(1-\gamma^{1/2})} \eptt\norm{\eptt\left[g_j\middle| \mathcal{F}_i\right]-\gdt F(\theta_t, \alpha_t)}_2\nonumber\\
	&\quad \le \frac{2R_{max}\sqrt{|\Ac|}}{(1-\gamma)(1-\gamma^{1/2})} \eptt \lt{\sqrt{\sum_{s,a} \left(\prob\left\{s_j =s|s_i\right\}\frac{Q(s,a)}{1-\gamma} - d_{\pi_{\theta_t}} (s) \frac{Q(s,a)}{1-\gamma}\right)^2}}\nonumber\\
	&\quad\le \frac{2R^2_{max}\sqrt{|\Ac|}}{(1-\gamma)^3(1-\gamma^{1/2})}\sqrt{\sum_{s,a} \left(\prob\left\{s_j =s|s_i\right\} - d_{\pi_{\theta_t}} (s)\right)^2}\nonumber\\
	&\quad\overset{(i)}=\frac{2R^2_{max}{|\Ac|}}{(1-\gamma)^3(1-\gamma^{1/2})}\norm{\prob\left\{s_j =\cdot|s_i\right\} - \chi_{\pi_{\theta_t}}}_2\nonumber\\
	&\quad\overset{(ii)}\le\frac{4C_M R^2_{max}{|\Ac|}}{(1-\gamma)^3(1-\gamma^{1/2})} \rho^{j-i}, \label{eq:midpgerrorbound3}
	\end{align}
	where $(i)$ follows because $\chi_{\pi_{\theta_t}} = d_{\pi_{\theta_t}}$, and $(ii)$ follows from \Cref{aspt:ergodic} and because  $d_{\pi_{\theta_t}}  = \chi_{\theta_t}$ and $$\norm{\prob\left\{s_j =\cdot|s_i\right\} - d_{\pi_{\theta_t}}}_2\le \norm{\prob\left\{s_j =\cdot|s_i\right\} - d_{\pi_{\theta_t}}}_1 = 2\dtv\left(\prob\left\{s_j =\cdot|s_i\right\}, d_{\pi_{\theta_t}}\right).$$ 
	
	Substituting \cref{eq:midpgerrorbound3} into \cref{eq:midpgerrorbound}, we obtain
	\begin{align*}
	\eptt\left[\norm{\sgdt F(\theta_t, \alpha_t)-\gdt F(\theta_t, \alpha_t)}_2^2\right] &\le\frac{4|\Ac|R_{max}^2 }{b(1-\gamma^{1/2})^2(1-\gamma)^2} + \frac{2}{b^2} \sum_{i=1}^{b-2} \sum_{j=i+1}^{b-1} \frac{4C_M|\Ac|R_{max}^2}{(1-\gamma^{1/2})^2(1-\gamma)^2} \rho^{j-i} \\
	&\le \frac{4|\Ac|R_{max}^2 }{b(1-\gamma^{1/2})^2(1-\gamma)^2}  \left(1+\frac{2C_M\rho}{1-\rho}\right)\frac{1}{b}. 
	\end{align*}
	
	The second claim can be easily checked. 
\end{proof}	
\subsection{Proof of \Cref{thm:ppg}}
	Based on the projection property, we have 
	\begin{equation}\label{eq:projectionproperty}
	\inner{\theta_t- \eta \sgdt F(\theta_{t}, \alpha_t) - \theta_{t+1}}{\theta - \theta_{t+1}} \le 0,\quad \forall \theta \in \Theta.
	\end{equation}
	Next we use \cref{eq:projectionproperty} to upper bound on $\eptt\left[\lt{\theta_{t+1} -\theta_t}^2\right]$. Letting $\theta = \theta_t$ and rearranging \cref{eq:projectionproperty} yield
	\begin{equation}\label{eq:upperboundofinnerproduct}
	\inner{\sgdt F(\theta_{t}, \alpha_t)}{\theta_{t+1} - \theta_{t}} \le -\eta^{-1} \|\theta_{t+1} -\theta_{t} \|_2^2.  
	\end{equation}
	
	According to the gradient Lipschitz condition established in \Cref{lemma:gradientlipschitz}, we have
	\begin{align*}
	g(\theta_{t+1}) &\le g(\theta_t) + \inner{\gdt g(\theta_t)}{\theta_{t+1} - \theta_t} + \left(\frac{L_{11}}{2} + \frac{L_{12}L_{21}}{2\mu}\right) \|\theta_{t+1} -\theta_{t} \|_2^2 \\ 
	&= g(\theta_t) +  \inner{\sgdt F(\theta_t, \alpha_t)}{\theta_{t+1} - \theta_{t}}- \inner{\gdt F(\theta_t, \alpha_t)  - \gdt g(\theta_t)}{\theta_{t+1} - \theta_{t}}\\
	&\quad - \inner{\sgdt F(\theta_t, \alpha_t) - \gdt F(\theta_t, \alpha_t)}{\theta_{t+1} - \theta_t} + \left(\frac{L_{11}}{2} + \frac{L_{12}L_{21}}{2\mu}\right) \|\theta_{t+1} -\theta_{t} \|_2^2\\
	&\overset{(i)}\le g(\theta_t) - \left(\frac{L_{11}}{2} + \frac{L_{12}L_{21}}{2\mu}\right) \|\theta_{t+1} -\theta_{t} \|_2^2 - \inner{\gdt F(\theta_t, \alpha_t)  - \gdt g(\theta_t)}{\theta_{t+1} - \theta_{t}}\\
	&\quad- \inner{\sgdt F(\theta_t, \alpha_t) - \gdt F(\theta_t, \alpha_t)}{\theta_{t+1} - \theta_t},
	\end{align*}
	where $(i)$ follows from \cref{eq:upperboundofinnerproduct} and the fact that $\eta = \left(L_{11} + \frac{L_{12}L_{21}}{\mu}\right)^{-1}$.
	
	Rearranging the above inequality, we obtain
	\begin{align*}
	\|\theta_{t+1} -\theta_{t} \|_2^2	&\le \left(\frac{L_{11}}{2} + \frac{L_{12}L_{21}}{2\mu}\right)^{-1} \left(g(\theta_{t}) - g(\theta_{t+1}) \right) \\
	&\quad - \left(\frac{L_{11}}{2} + \frac{L_{12}L_{21}}{2\mu}\right)^{-1}\inner{\gdt F(\theta_t, \alpha_t)  - \gdt g(\theta_t)}{\theta_{t+1} - \theta_{t}} \nonumber\\
	&\quad  -\left(\frac{L_{11}}{2} + \frac{L_{12}L_{21}}{2\mu}\right)^{-1}\inner{\sgdt F(\theta_t, \alpha_t) - \gdt F(\theta_t, \alpha_t)}{\theta_{t+1} - \theta_t}\nonumber\\
	&\overset{(i)}\le \left(\frac{L_{11}}{2} + \frac{L_{12}L_{21}}{2\mu}\right)^{-1} \left(g(\theta_{t}) - g(\theta_{t+1}) \right) \\
	&\quad +  \left(\frac{L_{11}}{2} + \frac{L_{12}L_{21}}{2\mu}\right)^{-2}\|\gdt F(\theta_t, \alpha_t)  - \gdt g(\theta_t)\|_2^2 + \frac{1}{4}\|\theta_{t+1} -\theta_{t} \|_2^2  \nonumber\\
	&\quad + \left(\frac{L_{11}}{2} + \frac{L_{12}L_{21}}{2\mu}\right)^{-2}\|\sgdt F(\theta_t, \alpha_t) - \gdt F(\theta_t, \alpha_t)\|_2^2 +\frac{1}{4}\|\theta_{t+1} -\theta_{t} \|_2^2,
	\end{align*}
	where  $(i)$ follows from Young's inequality.
	
	Taking expectation on both sides of the above inequality yields
	\begin{align}
	\eptt\left[\|\theta_{t+1} -\theta_{t} \|_2^2\right] &\overset{(i)}\le \frac{4\mu}{\mu L_{11} + L_{12}L_{21}}\eptt\left[g(\theta_{t}) - g(\theta_{t+1}) \right] + \frac{8\mu^2L_{22}^2}{(\mu L_{11} + L_{12}L_{21})^2}\eptt\left[\norm{\alpha_t - \alpha_{op}(\theta_{t})}_2^2\right] \nonumber\\
	&\quad  + \frac{8\mu^2}{(\mu L_{11} + L_{12}L_{21})^2}\eptt\left[\norm{\sgdt F(\theta_t, \alpha_t) - \gdt F(\theta_t, \alpha_t)}_2^2\right],\label{eq:telemo}
	\end{align}
	where $(i)$ follows from the gradient Lipschitz condition established in \Cref{lemma:lipschitzcondtion}
	
	Next, rearranging \cref{eq:projectionproperty}, we obtain
	\begin{align*}
		&\inner{\theta_t - \theta_{t+1}}{\theta - \theta_{t+1}}\\
		&\quad\le \eta \inner{\sgdt F(\theta_{t}, \alpha_t)}{\theta - \theta_{t+1}}\\
		&\quad= \eta \inner{\sgdt F(\theta_{t}, \alpha_t) - \gdt F(\theta_{t}, \alpha_t)}{\theta - \theta_{t+1}}+ \eta \inner{\gdt F(\theta_{t}, \alpha_t) - \gdt g(\theta_{t})}{\theta - \theta_{t+1}} \\
		&\qquad + \eta \inner{\gdt g(\theta_{t}, \alpha_t)}{\theta - \theta_{t}} + \eta \inner{\gdt g(\theta_{t}, \alpha_t)}{\theta_{t} - \theta_{t+1}}.
	\end{align*}
	Letting $\eta = \left(L_{11} + \frac{L_{12}L_{21}}{\mu}\right)^{-1}$ and rearranging the above inequality yield
	\begin{align}
	\inner{\gdt g(\theta_t)}{\theta - \theta_t} &\ge \left(L_{11} + \frac{L_{12}L_{21}}{\mu}\right) \inner{\theta_t - \theta_{t+1}}{\theta - \theta_{t+1}} - \inner{\gdt F(\theta_t, \alpha_t)  - \gdt g(\theta_t)}{\theta - \theta_{t+1}} \nonumber\\
	&\quad- \inner{\sgdt F(\theta_t, \alpha_t) - \gdt F(\theta_t, \alpha_t)}{\theta - \theta_{t+1}} - \inner{\gdt g(\theta_t)}{\theta_t - \theta_{t+1}}\nonumber\\
	&\overset{(i)}\ge -\left(L_{11} + \frac{L_{12}L_{21}}{\mu}\right) \|\theta_t - \theta_{t+1} \|_2\cdot 2R - \frac{\sqrt{|\Ac|}R_{max}}{(1-\gamma)^2} \|\theta_{t+1} - \theta_t\|_2 \nonumber\\
	&\quad - 2R(\|\sgdt F(\theta_t, \alpha_t) - \gdt F(\theta_t, \alpha_t)\|_2 + \|\gdt F(\theta_t, \alpha_t)  - \gdt g(\theta_t)\|_2) ,\label{eq:midppg}
	\end{align}
	where $(i)$ follows from the Cauchy-Schwartz inequality and the boundness properties of $\Theta_p$ ($R\defeq\max_{\theta\in\Theta_p}\{\lt{\theta}\}$) and because $\norm{\gdt g(\theta_t)}_2 = \norm{\gdt F(\theta_t, \ap(\theta_t))}_2 \le \frac{\sqrt{|\Ac|}R_{max}}{(1-\gamma)^2} $.
	
	Applying the gradient dominance property of $g(\theta)$ established in \Cref{prop:PL-like}, we obtain
	\begin{align*}
		g(\theta_t) -g(\theta^*) &\le C_d \max_{\theta\in\Theta} \inner{\gdt g(\theta_t)}{\theta_t -\theta} \\
		&\overset{(i)}\le C_d\left(\frac{2(\mu L_{11}+ L_{12}L_{21})R}{\mu} + \frac{\sqrt{|\Ac|}R_{max}}{(1-\gamma)^2}\right) \|\theta_t - \theta_{t+1}\|_2 \\
		&\qquad+ 2RC_d\|\sgdt F(\theta_t, \alpha_t) - \gdt F(\theta_t, \alpha_t)\|_2+ 2RC_d\|\gdt F(\theta_t, \alpha_t)- \gdt g(\theta_t)\|_2,
	\end{align*}
	where $(i)$ follows by multiplying $-1$ on both sides of \cref{eq:midppg} and taking the maximum over all $\theta \in \Theta_p$.
	
	Taking expectation on both sides of above inequality and telescoping, we have 
	\begin{align*}
	&\frac{1}{T}\sum_{t=0}^{T-1}\eptt\left[g(\theta_t)\right]-g(\theta^*)\\
	&\quad\le C_d\left(\frac{2(\mu L_{11}+ L_{12}L_{21})R}{\mu} +\frac{\sqrt{|\Ac|}R_{max}}{(1-\gamma)^2}\right) 	\frac{1}{T}\sum_{t=0}^{T-1}\eptt\left[\|\theta_t - \theta_{t+1}\|_2\right]\nonumber\\
	&\qquad + 2RC_d\frac{1}{T}\sum_{t=0}^{T-1}\eptt\left[\|\sgdt F(\theta_t, \alpha_t) - \gdt F(\theta_t, \alpha_t)\|_2 \right] + 2RC_d 	\frac{1}{T}\sum_{t=0}^{T-1}\eptt\left[\|\gdt F(\theta_t, \alpha_t)  - \gdt g(\theta_t)\|_2\right] \nonumber\\
	&\quad\overset{(i)}\le C_d\left(\frac{2(\mu L_{11}+ L_{12}L_{21})R}{\mu} +  \frac{\sqrt{|\Ac|}R_{max}}{(1-\gamma)^2}\right)\sqrt{\eptt\left[\frac{1}{T} \sum_{t=0}^{T-1} \lt{\theta_{t} - \theta_{t+1}}^2\right]} \nonumber\\
	&\qquad + 2RC_d\frac{1}{T}\sum_{t=0}^{T-1}\eptt\left[\|\sgdt F(\theta_t, \alpha_t) - \gdt F(\theta_t, \alpha_t)\|_2 \right] + 2RC_d 	\frac{1}{T}\sum_{t=0}^{T-1}\eptt\left[\|\gdt F(\theta_t, \alpha_t)  - \gdt g(\theta_t)\|_2\right] \nonumber\\
	&\quad\overset{(ii)}\le \left(\frac{2(\mu L_{11}+ L_{12}L_{21})R}{\mu} + \frac{\sqrt{|\Ac|}R_{max}}{(1-\gamma)^2}\right)C_d \sqrt{\frac{4\mu}{\mu L_{11} + L_{12}L_{21}}\frac{\eptt\left[g(\theta_{0}) - g(\theta_{T}) \right]}{T}}\\ 
	& \qquad +\left(\frac{2(\mu L_{11}+ L_{12}L_{21})R}{\mu} + \frac{\sqrt{|\Ac|}R_{max}}{(1-\gamma)^2}\right)C_d \sqrt{ \frac{8\mu^2L_{22}^2}{(\mu L_{11} + L_{12}L_{21})^2}\eptt\left[\norm{\alpha_t - \alpha_{op}(\theta_{t})}_2^2\right]}\\
	&\qquad +\left(\frac{2(\mu L_{11}+ L_{12}L_{21})R}{\mu} + \frac{\sqrt{|\Ac|}R_{max}}{(1-\gamma)^2}\right)C_d \sqrt{\frac{8\mu^2}{(\mu L_{11} + L_{12}L_{21})^2}\eptt\left[\norm{\sgdt F(\theta_t, \alpha_t) - \gdt F(\theta_t, \alpha_t)}_2^2\right]}\\
	&\qquad + 2RC_d\frac{1}{T}\sum_{t=0}^{T-1}\eptt\left[\|\sgdt F(\theta_t, \alpha_t) - \gdt F(\theta_t, \alpha_t)\|_2 \right] + 2RC_d 	\frac{1}{T}\sum_{t=0}^{T-1}\eptt\left[\|\gdt F(\theta_t, \alpha_t)  - \gdt g(\theta_t)\|_2\right] \nonumber\\
	&\quad\overset{(iii)}\le \left(\frac{2(\mu L_{11}+ L_{12}L_{21})R}{\mu} + \frac{\sqrt{|\Ac|}R_{max}}{(1-\gamma)^2}\right)C_d\sqrt{\frac{4\mu}{\mu L_{11} + L_{12}L_{21}}\frac{R_{max}}{(1-\gamma)T}}\\
	&\qquad +\left(\frac{\sqrt{|\Ac|}R_{max}}{(1-\gamma)^2}\frac{2\mu}{\mu L_{11} + L_{12}L_{21}} + 5R\right)2L_{22}C_d\sqrt{C_\alpha^2 e^{-\frac{\mu^2}{8L_{22}^2} K} + \frac{48C_r^2}{\mu^2(1-\gamma)^2}(1+\frac{C_M}{1-\rho})\frac{1}{B}} \\
	&\qquad +\left(\frac{\sqrt{|\Ac|}R_{max}}{(1-\gamma)^2}\frac{2\mu}{\mu L_{11} + L_{12}L_{21}} + 5R\right)2C_d\sqrt{\frac{4|\Ac|R_{max}^2 }{b(1-\gamma^{1/2})^2(1-\gamma)^2}  \left(1+\frac{2C_M\rho}{1-\rho}\right)\frac{1}{b}}\\
	&\quad\overset{(iv)}\le  \bigO{\frac{1}{(1-\gamma)^3 \sqrt{T}}} + \bigO{e^{-(1-\gamma)^2K}}+ \bigO{\frac{1}{(1-\gamma)^3 \sqrt{B}}} + \bigO{\frac{1}{(1-\gamma)^3\sqrt{b}}},
	\end{align*}
	where $(i)$ follows because $\eptt\left[X\right]\le \sqrt{\eptt\left[X^2\right]}$ holds for any random variable $X$, $(ii)$ follows by telescoping \cref{eq:telemo} and further because $\sqrt{a+b}\le \sqrt{a} +\sqrt{b}$ holds, for all $a, b >0$, $(iii)$ follows from \Cref{lemma:pgerrorbound,thm:alphaupdate} and because $\eptt\left[X\right]\le \sqrt{\eptt\left[X^2\right]}$ holds for any random variable $X$, and $(iv)$ follows because $L_{11} =\bigO{\frac{1}{(1-\gamma)^2}}$, $L_{12}= \bigO{\frac{1}{(1-\gamma)^2}}$, $L_{21}= \bigO{\frac{1}{1-\gamma}}$, $L_{22}= \bigO{\frac{1}{1-\gamma}}$, $C_d = \bigO{\frac{1}{1-\gamma}}$ and $\bigO{\frac{1}{1-\gamma^{1/2}}} \le \bigO{{\frac{1}{1-\gamma}}}$.
  	
\subsection{Proof of \Cref{thm:fwalg}}
By the gradient Lipschitz condition (established in \Cref{lemma:gradientlipschitz}) of $g(\theta)$, we have
\begin{align}
	g(\theta_{t+1}) &\le g(\theta_t) + \inner{\gdt g(\theta_t)}{\theta_{t+1} - \theta_{t}} + \left(\frac{L_{11}}{2} + \frac{L_{12}L_{21}}{2\mu}\right)\norm{\theta_{t+1} -\theta_t}_2^2 \nonumber\nonumber\\
	&= g(\theta_t) + \eta\inner{\gdt g(\theta_t)}{\hat{v_t} - \theta_t} + \left(\frac{L_{11}}{2} + \frac{L_{12}L_{21}}{2\mu}\right)\eta^2\norm{\hat{v}_t -\theta_t}_2^2\nonumber\\
	&\overset{(i)}\le g(\theta_t) + \eta\inner{\sgdt F(\theta_t, \alpha_t)}{\hat{v_t} -\theta_t} + \eta\inner{\gdt g(\theta_t) -\sgdt F(\theta_t, \alpha_t)}{\hat{v}_t - \theta_t} \nonumber\\
	&\quad + \left(2 L_{11} + \frac{2L_{12}L_{21}}{\mu}\right)\eta^2R^2\nonumber\\
	&\overset{(ii)}\le g(\theta_t) + \eta\inner{\sgdt F(\theta_t, \alpha_t)}{v_t -\theta_t} + \eta \inner{\gdt g(\theta_t) -\sgdt F(\theta_t, \alpha_t)}{\hat{v}_t - \theta_t} \nonumber\\
	&\quad + \left(2 L_{11} + \frac{2L_{12}L_{21}}{\mu}\right)\eta^2R^2\nonumber\\
	& = g(\theta_t) + \eta\inner{\gdt g(\theta_t)}{v_t - \theta_t} + \eta \inner{\gdt g(\theta_t) - \sgdt F(\theta_t, \alpha_t)}{\hat{v}_t - v_t}\nonumber \\
	&\quad + \left(2 L_{11} + \frac{2L_{12}L_{21}}{\mu}\right)\eta^2R^2,\label{eq:midfwpg}
\end{align}
where $(i)$ follows because $\norm{\hat{v}_t -\theta_t}_2\le 2R$, and $(ii)$ follows by definition of $\hat{v}_t$ in \cref{eq:fwupdate} (recall that $\hat{v}_t \defeq \argmax_{\theta\in\Theta_p} \langle\theta, -\sgdt F(\theta_{t}, \alpha_{t})\rangle$), and further we define $v_t \defeq  \argmax_{\theta\in\Theta} \inner{\theta}{-\gdt g(\theta_t)}$. We continue the proof as follows:
\begin{align}
	&\max_{\theta\in\Theta} \inner{\gdt g(\theta_t)}{\theta_t - \theta}\nonumber\\
	&\quad \overset{(i)}= \inner{\gdt g(\theta_t)}{\theta_t - v_t} \nonumber\\
	&\quad\overset{(ii)}\le \eta^{-1} \left(g(\theta_t) - g(\theta_{t+1})\right) +\left(2 L_{11} + \frac{2L_{12}L_{21}}{\mu}\right)\eta R^2\nonumber\\ 
	&\qquad + \inner{\gdt g(\theta_t) - \gdt F(\theta_t, \alpha_t)}{\hat{v}_t - v_t} + \inner{\gdt F(\theta_t, \alpha_t) - \sgdt F(\theta_t, \alpha_t)}{\hat{v}_t - v_t} \nonumber\\
	&\quad\le \eta^{-1} \left(g(\theta_t) - g(\theta_{t+1})\right) + \left(2 L_{11} + \frac{2L_{12}L_{21}}{\mu}\right)\eta R^2 \nonumber\\
	&\qquad +2R\norm{\gdt g(\theta_t) - \gdt F(\theta_t, \alpha_t)}_2 + 2R\norm{\gdt F(\theta_t, \alpha_t) - \sgdt F(\theta_t, \alpha_t)}_2,\label{eq:midfwpg2}
\end{align}
where $(i)$ follows by definition $v_t \defeq  \argmax_{\theta\in\Theta} \inner{\theta}{-\gdt g(\theta_t)}$, and $(ii)$ follows by rearranging \cref{eq:midfwpg}.

Finally, we complete the proof as follows:
\begin{align*}
	&\frac{1}{T}\sum_{t=0}^{T-1}\eptt \left[g(\theta_t)\right] - g(\theta^*)\nonumber\\
	&\quad\overset{(i)}\le C_d\cdot\frac{1}{T}\sum_{t=0}^{T-1}\eptt\left[\max_{\theta\in\Theta} \inner{\gdt g(\theta_t)}{\theta_t - \theta}\right] \nonumber\\	
	&\quad\overset{(ii)}\le \frac{C_d \eptt\left[g(\theta_0) - g(\theta_T)\right]}{\eta T} + C_d \left(2 L_{11} + \frac{2L_{12}L_{21}}{\mu}\right)\eta R^2+ \frac{2RC_d}{T}\sum_{t= 0}^{T-1}\eptt\norm{\gdt g(\theta_t) - \gdt F(\theta_t, \alpha_t)}_2\\
	&\qquad+ \frac{2RC_d}{T}\sum_{t= 0}^{T-1}\eptt\norm{\gdt F(\theta_t, \alpha_t) - \sgdt F(\theta_t, \alpha_t)}_2 \\
	&\quad\overset{(iii)}\le C_d\cdot\frac{ R_{max} + 2(1-\gamma)^{3}\left( L_{11} + L_{12}L_{21}\mu^{-1}\right)R^2}{(1-\gamma)^{2}\sqrt{T}} + 2RC_d\sqrt{\frac{4|\Ac|R_{max}^2 }{b(1-\gamma^{1/2})^2(1-\gamma)^2}  \left(1+\frac{2C_M\rho}{1-\rho}\right)\frac{1}{b}} \\
	&\qquad + 2RC_dL_{22}\sqrt{C_\alpha^2 e^{-\frac{\mu^2}{8L_{22}^2} K} + \frac{48C_r^2}{(1-\gamma)^2\mu^2}(1+\frac{C_M}{1-\rho})\frac{1}{B}}\\
	&\quad \overset{(iv)}\le \bigO{\frac{1}{(1-\gamma)^{3} \sqrt{T}}} + \bigO{e^{-(1-\gamma)^2K}}+ \bigO{\frac{1}{(1-\gamma)^3 \sqrt{B}}} + \bigO{\frac{1}{(1-\gamma)^3\sqrt{b}}}, 
\end{align*}
where $(i)$ follows from \Cref{prop:PL-like}, $(ii)$ follows from telescoping \cref{eq:midfwpg2}, $(iii)$ follows from \Cref{thm:alphaupdate,lemma:pgerrorbound} and because  $\eta = \frac{1-\gamma}{\sqrt{T}}$ and $\eptt\left[X\right]\le \sqrt{\eptt\left[X^2\right]}$ holds for any random variable $X$, and $(iv)$ follows because $L_{11} =\bigO{\frac{1}{(1-\gamma)^2}}$, $L_{12}= \bigO{\frac{1}{(1-\gamma)^2}}$, $L_{21}= \bigO{\frac{1}{1-\gamma}}$, $L_{22}= \bigO{\frac{1}{1-\gamma}}$, $C_d = \bigO{\frac{1}{1-\gamma}}$ and $\bigO{\frac{1}{1-\gamma^{1/2}}} \le \bigO{\frac{1}{1-\gamma}}$.


\section{Proof of \Cref{thm:trpo1,thm:trpo2}: Global Convergence of TRPO-GAIL}
In this section, we add the subscript $\lambda$ to the notations of the Q-function $Q^\pi_\alpha(s,a)$, the value function $V(\pi, r_\alpha)$, the objective function $F(\theta, \alpha)$ and $g(\theta)$ in order to emphasize that these functions are derived under $\lambda$-regularized MDP.

\subsection{Supporting Lemmas}
In this subsection, we introduce several useful lemmas.
\begin{lemma}(\textrm{\cite[Lemma 9.1]{beck2017first}})\label{lemma:threepoint}
	Consider a proper closed convex function $\omega$: $E \to (-\infty, \infty]$. Let $dom(\partial\omega)$ denote the subset of $E$ where $\omega$ is differentiable and $dom(\omega)$ denote the subset of $E$ where the value of $\omega$ is finite. Assume $a, b\in dom(\partial\omega)$ and $c\in dom(\omega)$. Then the following inequality holds:
	\begin{equation*}
		\inner{\nabla\omega(b) - \nabla\omega(a)}{c-a} = B_\omega (c,a) + B_\omega(a,b) - B_\omega(c,b),
	\end{equation*}
	where $B_\omega(\cdot, \cdot)$ denotes the Bregman distance associated with $\omega(\cdot)$. 
\end{lemma}
\begin{lemma}(\cite[Lemma 25]{shani2020adaptive})\label{lemma:upperboundofdualnorm}
	Consider the Q-function estimation in \Cref{alg:EstQ}. For any $t\in \{0, 1, \cdots, T-1\}$, we have 
	\begin{align*}
		\norm{-\hat{Q}^{\pi_{\theta_t}}_{\lambda, \alpha_t}(s, \cdot) + \lambda \nabla\omega(\pi_{\theta_t}(\cdot|s))}_\infty \le C_{\omega}(t;\lambda),
	\end{align*}
	where $\hat{Q}^{\pi_{\theta_t}}_{\lambda, \alpha_t}$ is the Q-function estimated under the reward function $r_{\alpha_t}$ and policy $\pi_{\theta_t}$, and $C_{\omega}(t;\lambda)\le \bigO{\frac{C_rC_\alpha(1+ \indicator{\lambda\neq 0}\log t)}{1-\gamma^{1/2}}}$. 
\end{lemma}
\begin{lemma}\label{lemma:fundamentaltransition} 
	For any policy $\pi,\pi' \in \Delta_{\Ac}$ and $\alpha \in \Lambda$, the following equality holds, 
	\begin{align*}
		& (V_\lambda(\pi, r_{\alpha})-V_\lambda(\pi', r_{\alpha}))(1-\gamma)\\
		&= \sum_{s\in\Sc} d_{\pi'}(s)\left(\inner{-Q^{\pi}_{\lambda,\alpha}(s,\cdot) + \lambda\nabla\omega(\pi(\cdot|s))}{\pi'(\cdot|s) - \pi(\cdot|s)} + \lambda B_\omega(\pi'(\cdot|s), \pi(\cdot|s))\right),
	\end{align*}
	where $V_\lambda(\pi, r_\alpha)$ is the average value function under  $\lambda$-regularized MDP with the reward function $r_{\alpha}$ and  $d_{\pi^\prime}$ is the state visitation distribution of $\pi^\prime$.
\end{lemma}
\begin{proof}[Proof of \Cref{lemma:fundamentaltransition}]
	Following from \cite[Lemma 24]{shani2020adaptive}, for any $s\in \Sc$, we have 
	\begin{align}
		\inner{-Q^{\pi}_{\lambda,\alpha}(s,\cdot) + \lambda\nabla\omega(\pi(\cdot|s))}{\pi'(\cdot|s) - \pi(\cdot|s)} = -(T_\lambda^{\pi'} V^\pi_{\lambda, \alpha}(s) - V^\pi_{\lambda, \alpha}(s)) -  \lambda B_\omega(\pi'(\cdot|s), \pi(\cdot|s)),\label{eq:trpomid1}
	\end{align}
	where $T_\lambda^{\pi'}$ is the Bellman operator under $\lambda$-regularized MDP, i.e., 
	$$T_\lambda^{\pi'} V^\pi_{\lambda, \alpha}(s) = \sum_{a\in \Ac} \big(\pi'(a|s)r_{\alpha, \lambda}(s,a) + \sum_{s'\in\Sc} \mathsf{P}(s'|s,a)V^{\pi}_{\lambda, \alpha}(s')\big).$$
	
	Furthermore, we have
	\begin{align*}
		&V_\lambda(\pi', r_{\alpha}) - V_\lambda(\pi, r_{\alpha})\\
		&\quad=\sum_{s} \zeta(s) (V^{\pi'}_{\lambda, \alpha}(s) -V^{\pi}_{\lambda, \alpha}(s))  \\
		&\quad\overset{(i)}= \frac{1}{(1-\gamma)}\sum_{s\in\Sc}d_{\pi'}(s)(T_\lambda^{\pi'} V^\pi_{\lambda, \alpha}(s) - V^\pi_{\lambda, \alpha}(s))	\\
		&\quad\overset{(ii)}=-\frac{1}{1-\gamma} \sum_{s\in\Sc} d_{\pi'}(s)\left(\inner{-Q^{\pi}_{\lambda,\alpha}(s,\cdot) + \lambda\nabla\omega(\pi(\cdot|s))}{\pi'(\cdot|s) - \pi(\cdot|s)} + \lambda B_\omega(\pi'(\cdot|s), \pi(\cdot|s))\right),
	\end{align*}
	where $(i)$ follows from \cite[Lemma 29]{shani2020adaptive} and $(ii)$ follows by multiplying \cref{eq:trpomid1} by $d_{\pi'}(s)$ and take the summation over $\Sc$.
\end{proof}
\subsection{Proof of \Cref{thm:trpo1,thm:trpo2}}

Since the unregularized MDP can be viewed as a special case of the regularized MDP, i.e., $\lambda = 0$, in this subsection,  we first develop our proof for the general regularized MDP up to a certain step, and then specialize to the case with $\lambda =0$ for proving \Cref{thm:trpo1} and continue to keep $\lambda$ general for proving \Cref{thm:trpo2}.

To we start the proof, recall that the update of $\theta_t$ specified in \cref{eq:trpo} satisfies, 
\begin{align}
	\pi_{\theta_{t+1}}(\cdot|s) \in \argmin_{\pi \in \Delta_{\Ac}} (\underbrace{\inner{-\hat{Q}^{\pi_{\theta_t}}_{\lambda,\alpha_t}(s, \cdot) + \lambda \nabla \omega (\pi_{\theta_t}(\cdot | s))}{\pi - \pi_{\theta_t}(\cdot|s)} + \eta_t^{-1} B_{\omega} (\pi, \pi_{\theta_t}(\cdot |s))}_{\defeq f_0(\pi)}).\nonumber
\end{align}
Following from the first-order optimality condition, we have
\begin{align}
	\nabla_{\pi} f_0 (\pi_{\theta_{t+1}} (\cdot|s))^\top(\pi -\pi_{\theta_{t+1}}(\cdot|s)) \ge 0, \forall \pi \in \Delta_{\Ac},\nonumber
\end{align}
which together with the fact 
\begin{align*}
	\nabla_{\pi} f_0 (\pi) = -\hat{Q}^{\pi_{\theta_t}}_{\lambda,\alpha_t}(s, \cdot) + \lambda \nabla \omega (\pi_{\theta_t}(\cdot | s)) + \eta_t^{-1} (\nabla\omega(\pi) - \nabla\omega(\pi_{\theta_{t}}(\cdot | s))),
\end{align*}
implies that 
\begin{align}\label{eq:firstorderoptimality}
	\inner{-\hat{Q}^{\pi_{\theta_t}}_{\lambda,\alpha_t}(s, \cdot) + \lambda \nabla\omega(\pi_{\theta_t}(\cdot|s)) +\eta_t^{-1}(\nabla \omega(\pi_{\theta_{t+1}}(\cdot | s)) - \nabla\omega(\pi_{\theta_{t}}(\cdot | s))) }{\pi -\pi_{\theta_{t+1}}(\cdot|s)} \ge 0
\end{align}
holds for any $\pi$.

Taking $\pi = \pi_{\theta^*}(\cdot|s)$ in \cref{eq:firstorderoptimality}, we obtain
\begin{align}
	0 &\le\eta_t \inner{-\hat{Q}^{\pi_{\theta_t}}_{\lambda,\alpha_t}(s, \cdot) + \lambda \nabla\omega(\pi_{\theta_t}(\cdot|s))}{\pi_{\theta^*}(\cdot|s) - \pi_{\theta_t}(\cdot |s)} \nonumber\\
	&\quad + \eta_t \inner{-\hat{Q}^{\pi_{\theta_t}}_{\lambda,\alpha_t}(s, \cdot) + \lambda \nabla\omega(\pi_{\theta_t}(\cdot|s))}{\pi_{\theta_t}(\cdot |s) -\pi_{\theta_{t+1}}(\cdot |s)} \nonumber\nonumber\\
	&\quad + \inner{\nabla\omega(\pi_{\theta_{t+1}}(\cdot | s)) - \nabla\omega(\pi_{\theta_{t}}(\cdot | s)) }{\pi_{\theta^*}(\cdot|s)  - \pi_{\theta_{t+1}}(\cdot|s)}\nonumber\\
	& \overset{(i)}\le \eta_t \inner{-\hat{Q}^{\pi_{\theta_t}}_{\lambda,\alpha_t}(s, \cdot) + \lambda \nabla\omega(\pi_{\theta_t}(\cdot|s))}{\pi_{\theta^*}(\cdot|s)  - \pi_{\theta_t}(\cdot |s)} \nonumber \\
	&\quad + \frac{\eta_t^2\norm{-\hat{Q}^{\pi_{\theta_t}}_{\lambda,\alpha_t}(s, \cdot) + \lambda \nabla\omega(\pi_{\theta_t}(\cdot|s))}_\infty^2}{2} + \frac{\norm{\pi_{\theta_t}(\cdot |s) -\pi_{\theta_{t+1}}(\cdot |s)}_1^2}{2} \nonumber\\
	&\quad + B_{\omega}(\pi_{\theta^*}(\cdot|s) , \pi_{\theta_t}(\cdot| s)) - B_{\omega}(\pi_{\theta^*}(\cdot|s) , \pi_{\theta_{t+1}}(\cdot| s))  - B_{\omega}(\pi_{\theta_{t+1}}(\cdot| s), \pi_{\theta_t}(\cdot| s))\nonumber\\
	& \overset{(ii)}\le \eta_t \inner{-\hat{Q}^{\pi_{\theta_t}}_{\lambda,\alpha_t}(s, \cdot) + \lambda \nabla\omega(\pi_{\theta_t}(\cdot|s))}{\pi_{\theta^*}(\cdot|s)  - \pi_{\theta_t}(\cdot |s)} + \frac{\eta_t^2C_{\omega}(t;\lambda)^2}{2} \nonumber\\
	&\quad + B_{\omega}(\pi_{\theta^*}(\cdot|s) , \pi_{\theta_t}(\cdot| s)) - B_{\omega}(\pi_{\theta^*}(\cdot|s) , \pi_{\theta_{t+1}}(\cdot| s)), \label{eq:middletrpo}
\end{align}
where $(i)$ follows from H\"older's inequality and \Cref{lemma:threepoint}, and $(ii)$ follows from the \Cref{lemma:upperboundofdualnorm} and Pinsker's inequality given by $$\frac{\norm{\pi_{\theta_t}(\cdot |s) -\pi_{\theta_{t+1}}(\cdot |s)}_1^2}{2} \le \kl{\pi_{\theta_{t+1}}(\cdot| s)}{\pi_{\theta_t}(\cdot| s)} =  B_{\omega}(\pi_{\theta_{t+1}}(\cdot| s), \pi_{\theta_t}(\cdot| s)),$$
where $\kl{\cdot}{\cdot}$ denotes the KL-divergence. 

Taking expectation conditioned on $\mathcal{F}_t = \sigma(\theta_0, \theta_1, \cdots, \theta_t)$ over \cref{eq:middletrpo}, we have
\begin{align}\label{eq:fundamentalineqoftrpo}
	0&\le \eta_t \inner{-Q^{\pi_{\theta_t}}_{\lambda,\alpha_t}(s, \cdot) + \lambda \nabla\omega(\pi_{\theta_t}(\cdot|s))}{\pi_{\theta^*}(\cdot|s) - \pi_{\theta_t}(\cdot |s)} + \frac{\eta_t^2C_{\omega}(t;\lambda)^2}{2} \nonumber\\
	& \qquad+ B_{\omega}(\pi_{\theta^*}(\cdot|s) , \pi_{\theta_t}(\cdot| s))-\eptt\left[B_{\omega}(\pi_{\theta^*}(\cdot|s), \pi_{\theta_{t+1}}(\cdot| s))\middle| \mathcal{F}_t \right].
\end{align}

Since \cref{eq:fundamentalineqoftrpo} holds for any state, we multiply it by $d_{\pi_{\theta^*}}(s)$ for each state $s$ and take the summation over $\Sc$. Then we rearrange the resulting bound and obtain
\begin{align}
& \frac{\eta_t^2C_{\omega}(t;\lambda)^2}{2}+ \sum_{s\in\Sc}d_{\pi_{\theta^*}}(s)B_{\omega}(\pi_{\theta^*}(\cdot|s) , \pi_{\theta_t}(\cdot| s))-\sum_{s\in\Sc}d_{\pi_{\theta^*}}(s)\eptt\left[B_{\omega}(\pi_{\theta^*}(\cdot|s), \pi_{\theta_{t+1}}(\cdot| s))\middle| \mathcal{F}_t \right]\nonumber\\
&\quad\ge -\eta_t \sum_{s\in\Sc}d_{\pi_{\theta^*}}(s)\inner{-Q^{\pi_{\theta_t}}_{\lambda,\alpha_t}(s, \cdot) + \lambda \nabla\omega(\pi_{\theta_t}(\cdot|s))}{\pi_{\theta^*}(\cdot|s) - \pi_{\theta_t}(\cdot |s)}\nonumber\\
&\quad \overset{(i)}= \eta_t(1-\gamma)(V_\lambda(\pi_{\theta^*}, r_{\alpha_t}) - V_\lambda (\pi_{\theta_t}, r_{\alpha_t}))  + \eta_t\lambda \sum_{s\in\Sc}d_{\pi_{\theta^*}}(s)B_{\omega}(\pi_{\theta^*}(\cdot|s) , \pi_{\theta_t}(\cdot| s)),\label{eq:midtrpo}
\end{align}
where $(i)$ follows from applying \Cref{lemma:fundamentaltransition} with $\pi = \pi_{\theta_t}$ and $\pi' = \pi_{\theta^*}$. Rearranging  \cref{eq:midtrpo}, we obtain
\begin{align}
&V_\lambda(\pi_{\theta^*}, r_{\alpha_t}) - V_\lambda (\pi_{\theta_t}, r_{\alpha_t})\nonumber\\
&\quad\le \frac{1}{\eta_t (1-\gamma)} \sum_{s\in\Sc} d_{\pi_{\theta^*}}(s)(1-\lambda\eta_t)  \eptt\left[B_{\omega}(\pi_{\theta^*}(\cdot|s), \pi_{\theta_t}(\cdot| s))\right] \nonumber\nonumber\\
&\qquad - \frac{1}{\eta_t (1-\gamma)} \sum_{s\in\Sc} d_{\pi_{\theta^*}}(s) \eptt\left[B_{\omega}(\pi_{\theta^*}(\cdot|s), \pi_{\theta_{t+1}}(\cdot| s))\right] + \frac{\eta_t C_\omega(t, \lambda)^2}{2(1-\gamma)}.\label{eq:midtrpo2}
\end{align}

Furthermore, we proceed the proof as follows:
\begin{align}
	&\eptt\left[g_\lambda(\theta_t)\right] - g_{\lambda}(\theta^*) \nonumber\\
	&\quad=  \eptt\left[g_\lambda(\theta_t) - F_\lambda(\theta_t, \alpha_t)\right] + \eptt\left[ F_\lambda(\theta_t, \alpha_t)- g_{\lambda}(\theta^*)\right]\nonumber\\
	&\quad\overset{(i)}\le \eptt\left[g_\lambda(\theta_t) - F_\lambda(\theta_t, \alpha_t)\right] + \eptt\left[ F_{\lambda}(\theta_t, \alpha_t) - F_{\lambda}(\theta^*, \alpha_t)\right]\nonumber\\
	&\quad\overset{(ii)}= \eptt\left[g_\lambda(\theta_t) - F_\lambda(\theta_t, \alpha_t)\right] + \eptt\left[V_{\lambda}(\pi_{\theta^*}, \alpha_t)  - V_{\lambda}(\pi_{\theta_t}, \alpha_t)\right] \nonumber\\
	&\quad\overset{(iii)}\le  L_{22}^2\eptt\left[\norm{\alpha_t -\alpha_{op}(\theta_t)}_2^2\right]+ \frac{1}{\eta_t (1-\gamma)} \sum_{s\in\Sc} d_{\pi_{\theta^*}}(s)(1-\lambda\eta_t)  \eptt\left[B_{\omega}(\pi_{\theta^*}(\cdot|s), \pi_{\theta_t}(\cdot| s))\right] \nonumber\\
	&\qquad - \frac{1}{\eta_t (1-\gamma)} \sum_{s\in\Sc} d_{\pi_{\theta^*}}(s) \eptt\left[B_{\omega}(\pi_{\theta^*}(\cdot|s), \pi_{\theta_{t+1}}(\cdot| s))\right] + \frac{\eta_t C_\omega(t, \lambda)^2}{2(1-\gamma)}, \label{eq:importtrpo}
\end{align}
where $(i)$ follows because $g_{\lambda}(\theta^*) \ge F_\lambda(\theta^*, \ap(\theta_t))$, $(ii)$ follows from the definition of $F_\lambda(\theta, \alpha)$, and $(iii)$ follows from the gradient Lipschitz condition of $\alpha$ in \Cref{lemma:lipschitzcondtion} and \cref{eq:midtrpo2}.

Next, to prove \Cref{thm:trpo1}, we let $\lambda=0$ and recall $\eta_t = \frac{1-\gamma}{\sqrt{T}}$. Telescoping \cref{eq:importtrpo}, we obtain 
\begin{align*}
	&\frac{1}{T} \sum_{t=0}^{T-1} \eptt\left[g(\theta_t)\right] - g(\theta^*)\\
	&\quad\le \frac{1}{(1-\gamma)^2\sqrt{T}}\sum_{s\in\Sc} d_{\pi_{\theta^*}}(s) \eptt\left[B_{\omega}(\pi_{\theta^*}(\cdot|s), \pi_{\theta_{0}}(\cdot| s)) - B_{\omega}(\pi_{\theta^*}(\cdot|s), \pi_{\theta_{T}}(\cdot| s))\right] \\
	&\qquad + \frac{L_{22}^2}{T}\sum_{t=0}^{T-1}\eptt\left[\norm{\alpha_t -\alpha_{op}(\theta_t)}_2^2\right] + \frac{C_\omega^2}{2\sqrt{T}} \\
	&\quad\overset{(i)}\le L_{22}^2 C_\alpha^2 e^{-\frac{\mu^2}{8L_{22}^2} K} + \frac{48C_r^2L_{22}^2}{\mu^2(1-\gamma)^2}(1+\frac{C_M}{1-\rho})\frac{1}{B} + \frac{(1-\gamma)^2C_{\omega}^2 + 2\log |\Ac|}{2(1-\gamma)^2\sqrt{T}}\\
	&\quad\overset{(ii)}\le \bigO{\frac{1}{(1-\gamma)^2\sqrt{T}}} + \bigO{e^{-(1-\gamma)^2K}}+ \bigO{\frac{1}{(1-\gamma)^4 B}},
\end{align*}
where $(i)$ follows from \Cref{thm:alphaupdate} and because $0\le B_{\omega}(\pi_1,\pi_2)\le \log|\Ac|$ for any $\theta_1, \theta_2$ and $(ii)$ follows because $L_{22}= \bigO{\frac{1}{1-\gamma}}$ and $C_\omega = \bigO{\frac{1}{1-\gamma^{1/2}}}\le \bigO{\frac{1}{1-\gamma}}$. This completes the proof of \Cref{thm:trpo1}.

To prove the \Cref{thm:trpo2}, let $\eta_t = \frac{1}{\lambda(t+2)}$. Then, telescoping \cref{eq:importtrpo} and applying \Cref{thm:alphaupdate}, we obtain
\begin{align*}
		&\frac{1}{T} \sum_{t=0}^{T-1} \eptt\left[g_\lambda(\theta_t)\right] - g_\lambda(\theta^*)\\
		&\quad\le  L_{22}^2C_\alpha^2 e^{-\frac{\mu^2}{8L_{22}^2} K} + \frac{48C_r^2L_{22}^2}{\mu^2(1-\gamma)^2}(1+\frac{ C_M}{1-\rho})\frac{1}{B} + \frac{C_\omega^2(T,\lambda)}{2(1-\gamma)\lambda}\frac{\log(T+1)}{T}  \\
		&\qquad+\frac{\lambda \sum_{s}d_{\pi_{\theta^*}(s)}\eptt[B_\omega(\pi_{\theta^*}(\cdot|s), \pi_{\theta_{0}}(\cdot| s)) - (T+1)B_\omega(\pi_{\theta^*}(\cdot|s), \pi_{\theta_{T}}(\cdot| s))] }{(1-\gamma)T}\\
		&\quad\overset{(i)}\le \bigO{{\frac{1}{(1-\gamma)^3 T}}} + \bigO{e^{-(1-\gamma)^2K}}+ \bigO{\frac{1}{(1-\gamma)^4 B}},
\end{align*}
where $(i)$ follows because  $0\le B_{\omega}(\pi_1, \pi_2)\le \log(|\Ac|)$ for any $\pi_1, \pi_2$, $L_{22}= \bigO{\frac{1}{1-\gamma}}$ and $C_\omega(T,\lambda) = \tilde{\mathcal{O}}\left(\frac{1}{1-\gamma^{1/2}}\right)\le \tilde{\mathcal{O}}\left(\frac{1}{1-\gamma}\right)$. This completes the proof of \Cref{thm:trpo2}.
%

\section{Proof of \Cref{thm:npg}: Global Convergence of NPG-GAIL}


To prove the theorem, we first define some notations. Let  $\lambda_P \defeq \min_{\theta\in\Theta} \left\{\lambda_{min}(F(\theta) + \lambda I)\right\}$, 
$$W_{\theta,\alpha}^{\lambda*} \defeq (F(\theta) + \lambda I)^{-1} \eptt_{(s,a)\sim \nu_{\pi_{\theta}}} \left[A^{\pi_{\theta}}_{\alpha}(s,a)\nabla_{\theta}\log\pi_{\theta}(a|s)\right]$$ 
and $$W_{\theta, \alpha}^* \defeq F(\theta)^\dagger\eptt_{(s,a)\sim \nu_{\pi_{\theta}}} \left[A^{\pi_{\theta}}_{\alpha}(s,a)\nabla_{\theta}\log\pi_{\theta}(a|s)\right].$$ 
For brevity, we denote $W_t^{\lambda*} = W_{\theta_t, \alpha_t}^{\lambda*}$ and $W_t^* = W_{\theta_t, \alpha_t}^*$.
\subsection{Supporting Lemmas}
In this subsection, we give several useful lemmas.
\begin{lemma}(\cite[Lemma 3.2]{agarwal2019optimality})\label{lemma:performancedifference}
	For any policy $\pi$ and $\pi^\prime$ and reward function $r_\alpha$, we have 
	\begin{align*}
		V(\pi, r_\alpha) - V(\pi^\prime, r_{\alpha}) = \frac{1}{1-\gamma} \eptt_{s,a \sim \nu_\pi(s,a)}\left[A^{\pi^\prime}_{\alpha}(s,a)\right].
	\end{align*}
\end{lemma}
\begin{lemma}(\cite[Lemma 6]{xu2020improving})\label{lemma:differenceofregularization}
	For any $\theta$ and $\alpha$, we have $\lt{W_{\theta, \alpha}^{\lambda*} - W_{\theta, \alpha}^*} \le C_\lambda\lambda$, where $0<C_\lambda<\infty$ is a constant only depending on the policy class.
\end{lemma}
\begin{lemma}\label{lemma:linearsa}
	Suppose \Cref{aspt:ergodic,assp:generalpolicyparameterization} hold. Consider the policy update of NPG-GAIL (\Cref{alg:npggail}) with $\beta_W = \frac{\lambda_P}{4(C_\phi^2 +\lambda)^2}$. Then, for all $t = 0, 1, \cdots, T-1$, we have 
	\begin{align*}
	\mE[\ltwo{w_t - {W}_t^{\lambda*}}^2]&\leq \exp\left\{-\frac{\lambda_P^2 T_c}{16(C_\phi^2 +\lambda)^2}\right\}\frac{R_{max}^2C_\phi^2}{\lambda_P^2(1-\gamma)^2}\\
		&\quad + \parat{\frac{1}{\lambda_P} + \frac{\lambda_P}{2(C_\phi^2 +\lambda)^2}} \frac{98R_{max}^2 C_\phi^2[(C_\phi^2 +\lambda)^2 + 4\lambda_P^2][1 +(C_M -1)\rho]}{(1-\rho)(1-\gamma)^2\lambda_P^3 M}.
	\end{align*} 
\end{lemma} 
\begin{proof}[Proof of \Cref{lemma:linearsa}]
	At iteration $t$, $W_0, W_1, \cdots, W_{T_c}$ follows the linear SA iteration rule defined in \cite[eq. (3)]{xu2020improving} with $\alpha= \beta_W$, $A = - (F(\theta_t) + \lambda I)$, $b = \eptt_{(s,a)\sim \nu_{\pi_{\theta_t}}} \left[A^{\pi_{\theta_t}}_{\alpha_t}(s,a)\nabla_{\theta_t}\log\pi_{\theta_t}(a|s)\right]$ and $\theta^*= -A^{-1}b= W_t^{\lambda*}$ with $\lt{W_t^{\lambda*}}\le R_\theta = \frac{2C_\phi R_{max}}{\lambda_A(1-\gamma)}$.	It is easy to check that the Assumption 3 in \cite{xu2020improving} holds. Namely, $(i)$, $\norm{A}_F \le C_\phi^2 +\lambda$ and $\lt{b}\le \frac{2R_{max}C_\phi}{1-\gamma}$; $(ii)$, for any $w\in\mathbb{R}^d$, $\inner{w- W_t^{\lambda*}}{A(w- W_t^{\lambda*})}\le - \lambda_p \lt{w- W_t^{\lambda*}}^2$; $(iii)$, The ergodicity of MDP is assumed here.	Thus, applying \cite[Theorem 4]{xu2020improving} completes the proof.
\end{proof}
\subsection{Proof of \Cref{thm:npg}}
Define $D(\theta) = \eptt_{s\sim d_{\pi_{\theta^*}}}[\kl{\pi_{\theta^*}(\cdot|s)}{\pi_{\theta}(\cdot|s)}]$. Then we have
\begin{align}
D(\theta_{t})- D(\theta_{t+1})&= \eptt_{\nu_{\pi_{\theta^*}}}\left[\log(\pi_{\theta_{t+1}}(\cdot|s)) - \log(\pi_{\theta_{t}}(\cdot|s))\right] \nonumber\\
&\overset{(i)}\ge \eptt_{\nu_{\pi_{\theta^*}}} \left[\gdt \log(\pi_{\theta_{t}}(a|s))\right]^\top(\theta_{t+1} - \theta_{t}) -\frac{L_\phi^2}{2}\lt{\theta_{t+1} - \theta_{t}}^2,\nonumber
\end{align}
where $(i)$ follows from the gradient Lipschitz condition on $\log(\pi_{\theta}(\cdot|s))$ in \Cref{assp:generalpolicyparameterization}. 

Recall that the update rule in NPG-GAIL (\Cref{alg:npggail}) is  given by $\theta_{t+1} = \theta_t - \eta w_t$. Then we have 
\begin{align}
&D(\theta_{t})- D(\theta_{t+1})\nonumber\\
&\quad\ge \eta\eptt_{\nu_{\pi_{\theta^*}}} \left[\gdt \log(\pi_{\theta_{t}}(a|s))\right]^\top w_t - \frac{L_\phi^2\eta^2}{2}\lt{w_t}^2 \nonumber\\
&\quad=\eta\eptt_{\nu_{\pi_{\theta^*}}}\left[A^{\pi_{\theta_t}}_{\alpha_t}(s,a)\right]+ \eta\eptt_{\nu_{\pi_{\theta^*}}}\left[ \gdt \log(\pi_{\theta_{t}}(a|s))^\top W_t^*- A^{\pi_{\theta_t}}_{\alpha_t}(s,a)\right]\nonumber\\
&\qquad + \eta\eptt_{\nu_{\pi_{\theta^*}}} \left[\gdt \log(\pi_{\theta_{t}}(a|s))\right]^\top(W_t^{\lambda*}-W_t^*)+ \eta\eptt_{\nu_{\pi_{\theta^*}}} \left[\gdt \log(\pi_{\theta_{t}}(a|s))\right]^\top (w_t - W_t^{\lambda*})\nonumber \\
&\qquad  -\frac{L_\phi^2\eta^2}{2}\lt{w_t}^2\nonumber\\
&\quad\overset{(i)} = (1-\gamma)\eta \left(V(\pi_{\theta^*}, r_{\alpha_t}) - V(\pi_{\theta_{t}}, r_{\alpha_t})\right)+ \eta\eptt_{\nu_{\pi_{\theta^*}}}\left[ \gdt \log(\pi_{\theta_{t}}(a|s))^\top W_t^*- A^{\pi_{\theta_t}}_{\alpha_t}(s,a)\right]\nonumber\\
&\qquad + \eta\eptt_{\nu_{\pi_{\theta^*}}} \left[\gdt \log(\pi_{\theta_{t}}(a|s))\right]^\top(W_t^{\lambda*}-W_t^*)+ \eta\eptt_{\nu_{\pi_{\theta^*}}} \left[\gdt \log(\pi_{\theta_{t}}(a|s))\right]^\top (w_t - W_t^{\lambda*}) \nonumber\\
&\qquad  -\frac{L_\phi^2\eta^2}{2}\lt{w_t}^2\nonumber\\
&\quad \overset{(ii)}\ge (1-\gamma)\eta \left(V(\pi_{\theta^*}, r_{\alpha_t}) - V(\pi_{\theta_{t}}, r_{\alpha_t})\right) -\frac{L_\phi^2\eta^2}{2}\lt{w_t}^2 \nonumber\\
&\qquad + \eta\eptt_{\nu_{\pi_{\theta^*}}} \left[\gdt \log(\pi_{\theta_{t}}(a|s))\right]^\top(W_t^{\lambda*}-W_t^*)+ \eta\eptt_{\nu_{\pi_{\theta^*}}} \left[\gdt \log(\pi_{\theta_{t}}(a|s))\right]^\top (w_t - W_t^{\lambda*})\nonumber\\
&\qquad - \eta\sqrt{\eptt_{\nu_{\pi_{\theta^*}}}\left[ (\gdt \log(\pi_{\theta_{t}}(a|s))^\top W_t^*- A^{\pi_{\theta_t}}_{\alpha_t}(s,a))^2\right]} \nonumber\\
&\quad\overset{(iii)}\ge (1-\gamma)\eta \left(V(\pi_{\theta^*}, r_{\alpha_t}) - V(\pi_{\theta_{t}}, r_{\alpha_t})\right) -\frac{L_\phi^2\eta^2}{2}\lt{w_t}^2 \nonumber\\
&\qquad + \eta\eptt_{\nu_{\pi_{\theta^*}}} \left[\gdt \log(\pi_{\theta_{t}}(a|s))\right]^\top(W_t^{\lambda*}-W_t^*)+ \eta\eptt_{\nu_{\pi_{\theta^*}}} \left[\gdt \log(\pi_{\theta_{t}}(a|s))\right]^\top (w_t - W_t^{\lambda*}) \nonumber\\
&\qquad  - \eta\sqrt{C_d\eptt_{\nu_{\pi_{\theta_{t}}}}\left[ (\gdt \log(\pi_{\theta_{t}}(a|s))^\top W_t^*- A^{\pi_{\theta_t}}_{\alpha_t}(s,a))^2\right]}, \label{eq:middle}
\end{align} 
where $(i)$ follows from \Cref{lemma:performancedifference}, $(ii)$ follows from the concavity of $f(x) = \sqrt{x}$ and Jensen's inequality, and $(iii)$ follows from the fact that $(\gdt \log(\pi_{\theta_{t}}(a|s))^\top W_t^*- A^{\pi_{\theta_t}}_{\alpha_t}(s,a))^2 \ge 0$ and $\left\|\frac{\nu_{\pi_{\theta^*}}}{\nu_{\pi_{\theta_t}}}\right\|_\infty \le\frac{1}{(1-\gamma)\min\left\{\zeta(s)\right\}}\defeq C_d$. 

Continuing to bound \cref{eq:middle}, we have
\begin{align}
&D(\theta_{t})- D(\theta_{t+1})\nonumber\\ 
&\quad \overset{(i)}\ge (1-\gamma)\eta \left(V(\pi_{\theta^*}, r_{\alpha_t}) - V(\pi_{\theta_{t}}, r_{\alpha_t})\right) -\frac{L_\phi^2\eta^2}{2}\lt{w_t}^2 - \eta\sqrt{C_d}\zeta'\nonumber\\
&\qquad + \eta\eptt_{\nu_{\pi_{\theta^*}}} \left[\gdt \log(\pi_{\theta_{t}}(a|s))\right]^\top(W_t^{\lambda*}-W_t^*)+ \eta\eptt_{\nu_{\pi_{E}}} \left[\gdt \log(\pi_{\theta_{t}}(a|s))\right]^\top (w_t - W_t^{\lambda*}) \nonumber\\
&\quad \overset{(ii)}\ge (1-\gamma)\eta \left(V(\pi_{\theta^*}, r_{\alpha_t}) - V(\pi_{\theta_{t}}, r_{\alpha_t})\right)- \eta\sqrt{C_d}\zeta' - \eta C_\phi C_\lambda\lambda\nonumber\\
&\qquad - \eta C_\phi \lt{w_t - W_t^{\lambda*}}-\frac{L_\phi^2\eta^2}{2}\lt{w_t}^2 \nonumber\\
&\quad \overset{(iii)}\ge (1-\gamma)\eta \left(V(\pi_{\theta^*}, r_{\alpha_t}) - V(\pi_{\theta_{t}}, r_{\alpha_t})\right)- \eta\sqrt{C_d}\zeta'- \eta C_\phi C_\lambda\lambda\nonumber\\
&\qquad - \eta C_\phi \lt{w_t - W_t^{\lambda*}} -L_\phi^2\eta^2\lt{w_t - W_t^{\lambda*}}^2 - L_\phi^2\eta^2\lt{W_t^{\lambda*}}^2  \nonumber\\
&\quad \overset{(iv)}\ge (1-\gamma)\eta \left(V(\pi_{\theta^*}, r_{\alpha_t}) - V(\pi_{\theta_{t}}, r_{\alpha_t})\right)- \eta\sqrt{C_d}\zeta' - \eta C_\phi C_\lambda\lambda\nonumber\\
&\qquad - \eta C_\phi \lt{w_t - W_t^{\lambda*}} -L_\phi^2\eta^2\lt{w_t - W_t^{\lambda*}}^2 - \frac{L_\phi^2\eta^2}{\lambda_P^2}\lt{\gdt V(\theta_{t},r_{\alpha_t})}^2, \label{eq:lasteq}
\end{align}
where $(i)$ follows from the definition of $\zeta^\prime$ in the statement of \Cref{thm:npg},  $(ii)$ follows from the upper bound on $\lt{\gdt \pi_{\theta}(a|s)}$ in \Cref{assp:generalpolicyparameterization}, \Cref{lemma:differenceofregularization} and Cauchy-Schwartz inequality, $(iii)$ follows from the fact $\lt{A+B}^2 \le 2\lt{A}^2 +2\lt{B}^2$, and $(iv)$ follows from the definition of $W_t^{\lambda*}$ and because $\lambda_P I \preceq F(\theta_t)+ \lambda I$.

Rearranging \cref{eq:lasteq}, we obtain
\begin{align}
V(\pi_{\theta^*}, r_{\alpha_t}) - V(\pi_{\theta_{t}}, r_{\alpha_t})&\le \frac{D(\theta_t) -D(\theta_{t+1})}{\eta(1-\gamma)} + \frac{\sqrt{C_d}\zeta'}{1-\gamma} + \frac{ C_\phi C_\lambda\lambda}{1-\gamma}+ \frac{C_\phi}{1-\gamma} \lt{w_t - W_t^{\lambda*}}\nonumber\\
&\quad + \frac{L_\phi^2\eta}{1-\gamma}\lt{w_t - W_t^{\lambda*}}^2+ \frac{L_\phi^2\eta}{\lambda_P^2(1-\gamma)}\lt{\gdt V(\theta_{t},r_{\alpha_t})}^2. \label{eq:reallasteq}
\end{align}

Finally, we complete the proof as follows:
\begin{align*}
&\frac{1}{T} \sum_{t=0}^{T-1} \eptt\left[g(\theta_t)\right]  - g(\theta^*) \\
&\quad= \frac{1}{T} \sum_{t=0}^{T-1} \eptt\left[g(\theta_t) - F(\theta_t, \alpha_t)\right] +  \frac{1}{T} \sum_{t=0}^{T-1} \eptt\left[F(\theta_t, \alpha_t) - g(\theta^*)\right]\\
&\quad\overset{(i)}\le\frac{1}{T} \sum_{t=0}^{T-1} \eptt\left[g(\theta_t) - F(\theta_t, \alpha_t)\right] +  \frac{1}{T} \sum_{t=0}^{T-1}(F(\theta_t, \alpha_t) - F(\theta^*, \alpha_t))\nonumber \\
&\quad= \frac{1}{T} \sum_{t=0}^{T-1} \eptt\left[g(\theta_t) - F(\theta_t, \alpha_t)\right] + \frac{1}{T} \sum_{t=0}^{T-1}(V(\pi_{\theta^*}, r_{\alpha_t})  - V(\pi_{\theta_t}, r_{\alpha_t}))\nonumber\\
&\quad\overset{(ii)}\le \frac{1}{T} \sum_{t=0}^{T-1} \eptt\left[g(\theta_t) - F(\theta_t, \alpha_t)\right] +  \frac{D(\theta_{0})- D(\theta_{T})}{(1-\gamma)\eta T} + \frac{\sqrt{C_d}\zeta'}{1-\gamma}+ \frac{C_\phi C_\lambda\lambda}{1-\gamma} \nonumber\\
&\qquad+\frac{C_\phi }{(1-\gamma)T} \sum_{t=0}^{T-1}\lt{w_t - W_t^{\lambda*}}+\frac{L_\phi^2\eta}{(1-\gamma)T}\sum_{t=0}^{T-1}\lt{w_t - W_t^{\lambda*}}^2 +\frac{L_\phi^2\eta R_{max}^2C_\phi^2}{(1-\gamma)^3\lambda_P^2}\\
&\quad\overset{(iii)}\le L_{22}^2C_\alpha^2 e^{-\frac{\mu^2}{8L_{22}^2} K} + \frac{48C_r^2L_{22}^2}{\mu^2(1-\gamma)^2}(1+\frac{\rho C_M}{1-\rho})\frac{1}{B} + \frac{\eptt\left[D(\theta_{0})- D(\theta_{T})\right]}{(1-\gamma)^2\sqrt{T}} +\frac{\sqrt{C_d}\zeta'}{1-\gamma}+ \frac{C_\phi C_\lambda\lambda}{1-\gamma}\\
&\qquad+\frac{C_\phi }{(1-\gamma)T} \sum_{t=0}^{T-1}\lt{w_t - W_t^{\lambda*}}+\frac{L_\phi^2}{T^{3/2}}\sum_{t=0}^{T-1}\lt{w_t - W_t^{\lambda*}}^2 +\frac{L_\phi^2 R_{max}^2C_\phi^2}{(1-\gamma)^2\lambda_P^2\sqrt{T}}\\
 &\quad\overset{(iv)}\le L_{22}^2C_\alpha^2 e^{-\frac{\mu^2}{8L_{22}^2} K} + \frac{48C_r^2L_{22}^2}{\mu^2(1-\gamma)^2}(1+\frac{\rho C_M}{1-\rho})\frac{1}{B} + \frac{\eptt\left[D(\theta_{0})- D(\theta_{T})\right]}{(1-\gamma)^2\sqrt{T}} +\frac{\sqrt{C_d}\zeta'}{1-\gamma}+ \frac{C_\phi C_\lambda\lambda}{1-\gamma} \nonumber\\
 &\qquad+\frac{C_\phi}{(1-\gamma)}\resizebox{0.8\hsize}{!}{$\sqrt{\exp\left\{-\frac{\lambda_P^2 T_c}{16(C_\phi^2 +\lambda)^2}\right\}\frac{R_{max}^2C_\phi^2}{\lambda_P^2(1-\gamma)^2} + \parat{\frac{1}{\lambda_P} + \frac{\lambda_P}{2(C_\phi^2 +\lambda)^2}} \frac{98R_{max}^2 C_\phi^2[(C_\phi^2 +\lambda)^2 + 4\lambda_P^2][1 +(C_M -1)\rho]}{(1-\rho)(1-\gamma)^2\lambda_P^3 M}}$}\\\
 &\qquad+\frac{L_\phi^2}{\sqrt{T}} \left(\resizebox{0.8\hsize}{!}{$\exp\left\{-\frac{\lambda_P^2 T_c}{16(C_\phi^2 +\lambda)^2}\right\}\frac{R_{max}^2C_\phi^2}{\lambda_P^2(1-\gamma)^2} + \parat{\frac{1}{\lambda_P} + \frac{\lambda_P}{2(C_\phi^2 +\lambda)^2}} \frac{98R_{max}^2 C_\phi^2[(C_\phi^2 +\lambda)^2 + 4\lambda_P^2][1 +(C_M -1)\rho]}{(1-\rho)(1-\gamma)^2\lambda_P^3 M}$}\right)\\
&\qquad +\frac{L_\phi^2 R_{max}^2C_\phi^2}{(1-\gamma)^2\lambda_P^2\sqrt{T}}\\
&\quad\overset{(v)}\le \bigO{\frac{1}{(1-\gamma)^2 \sqrt{T}}} + \bigO{e^{-(1-\gamma)^2K}}+ \bigO{\frac{1}{(1-\gamma)^4 B}}\\
&\qquad +\bigO{\frac{\zeta^\prime}{(1-\gamma)^{3/2}}}+\bigO{\frac{\lambda}{1-\gamma}}+ \bigO{e^{-T_c}} + \bigO{\frac{1}{(1-\gamma)^2\sqrt{M}}}, 	
\end{align*}
where $(i)$ follows because $g(\theta^*) = F(\theta^*, \alpha_{op}(\theta^*)) \ge F(\theta^*, \alpha_t)$ and $(ii)$ follows from \cref{eq:reallasteq} and because $\lt{\gdt V(\theta_{t},\alpha_{t})} \le \frac{R_{max}C_\phi}{1-\gamma}$, $(iii)$ follows from \Cref{lemma:lipschitzcondtion} and \Cref{thm:alphaupdate}, and the fact $\eta = \frac{1-\gamma}{\sqrt{T}}$, $(iv)$ follows from \Cref{lemma:linearsa}, and $(v)$ follows because $L_{22}= \bigO{\frac{1}{1-\gamma}}$ and $C_d= {\mathcal{O}}\left(\frac{1}{1-\gamma }\right)$.